\documentclass[final,12pt]{colt2026}

\usepackage{graphicx}
\usepackage{xcolor}
\usepackage{enumerate}
\usepackage{color}
\definecolor{darkblue}{rgb}{0.0,0.0,0.2}
\definecolor{gray}{gray}{0.5}
\definecolor{lightred}{rgb}{1,0.6,0.6}
\definecolor{darkgreen}{rgb}{0,0.5,0}

\newcommand{\Comments}{0}
\let\todo\undefined 
\usepackage[colorinlistoftodos,textsize=tiny]{todonotes}
\newcommand{\mynote}[2]{\ifnum\Comments=1\textcolor{#1}{#2}\fi}
\newcommand{\mytodo}[2]{\ifnum\Comments=1
  \todo[linecolor=#1!80!black,backgroundcolor=#1,bordercolor=#1!80!black]{#2}\fi}

\ifnum\Comments=1               
\else
\fi

\newif\ifspacehack
\usepackage{natbib}
\usepackage{hyperref}
\hypersetup{
    colorlinks = blue,
    breaklinks,
    linkcolor = blue,
    citecolor = blue,
    urlcolor  = blue,
}
\usepackage{url} 
\usepackage{graphicx}
\usepackage{mathtools}
\usepackage{footnote}
\usepackage{float}
\usepackage{xspace}
\usepackage{multirow}
\usepackage{xcolor}
\usepackage{wrapfig}
\usepackage{framed}
\usepackage{bbm}
\usepackage[most]{tcolorbox}

\usepackage{footnote}
\usepackage{nicefrac}
\usepackage{makecell}
\usepackage{amssymb}
\usepackage{bm}
\makesavenoteenv{tabular}
\makesavenoteenv{table}

\renewcommand{\tilde}{\widetilde}
\renewcommand{\hat}{\widehat}

\def \R {\mathbb{R}}

\newcommand{\calL}{{\mathcal{L}}}

\newcommand{\calS}{{\mathcal{S}}}
\newcommand{\calT}{{\mathcal{T}}}

\newcommand{\Reg}{{\mathrm{Reg}}}

\newcommand{\E}{{\mathbb{E}}}

\newcommand{\pbar}{\bar{p}}

\newcommand{\inner}[1]{ \left\langle {#1} \right\rangle }

\newcommand{\diag}{\textrm{diag}}

\newcommand{\order}{\mathcal{O}}

\newcommand{\ones}{\boldsymbol{1}}

\newcommand{\Var}{\mathrm{Var}}

\DeclareMathOperator*{\argmin}{argmin}
\DeclareMathOperator*{\argmax}{argmax}

\newcommand{\field}[1]{\mathbb{#1}}

\newcommand{\fR}{\field{R}}

\newcommand{\fB}{\field{B}}

\newcommand{\otil}{\ensuremath{\tilde{O}}}

\newcommand{\rbr}[1]{\left(#1\right)}
\newcommand{\sbr}[1]{\left[#1\right]}
\newcommand{\cbr}[1]{\left\{#1\right\}}

\usepackage{lipsum,booktabs}
\usepackage{amsmath,mathrsfs,amssymb,amsfonts,bm,enumitem}
\usepackage{rotating}
\usepackage{pdflscape}
\usepackage{hyperref,url}
\hypersetup{
    colorlinks,
    breaklinks,
    linkcolor = blue,
    citecolor = blue,
    urlcolor  = blue,
}
\allowdisplaybreaks
\usepackage{appendix}
\usepackage{multirow,makecell}

\renewcommand{\tilde}{\widetilde}
\renewcommand{\hat}{\widehat}

\def \A {\mathcal{A}}

\def \E {\mathbb{E}}

\def \L {\mathcal{L}}

\def \O {\mathcal{O}}

\def \R {\mathbb{R}}

\def \Pr {\mathsf{Pr}}

\usepackage{mathtools}

\newtcolorbox{promptbox}[1][]{
  colback=blue!5!white, colframe=blue!75!black,
  fonttitle=\bfseries, title=Prompt,
  left=2mm, right=2mm, top=2mm, bottom=2mm,
  boxrule=0.5mm,  
  coltitle=black, 
  colbacktitle=blue!15!white, 
  breakable,      
  #1
}

\usepackage{graphicx,color}

\definecolor{wine_red}{RGB}{228,48,64}
\definecolor{DSgray}{cmyk}{0,1,0,0}

\usepackage{prettyref}
\newcommand{\pref}[1]{\prettyref{#1}}

\newcommand{\savehyperref}[2]{\texorpdfstring{\hyperref[#1]{#2}}{#2}}
\newrefformat{eq}{\savehyperref{#1}{Eq. \textup{(\ref*{#1})}}}
\newrefformat{eqn}{\savehyperref{#1}{Eq.~(\ref*{#1})}}
\newrefformat{lem}{\savehyperref{#1}{Lemma~\ref*{#1}}}
\newrefformat{event}{\savehyperref{#1}{Event~\ref*{#1}}}
\newrefformat{def}{\savehyperref{#1}{Definition~\ref*{#1}}}
\newrefformat{line}{\savehyperref{#1}{Line~\ref*{#1}}}
\newrefformat{thm}{\savehyperref{#1}{Theorem~\ref*{#1}}}
\newrefformat{tab}{\savehyperref{#1}{Table~\ref*{#1}}}
\newrefformat{corr}{\savehyperref{#1}{Corollary~\ref*{#1}}}
\newrefformat{cor}{\savehyperref{#1}{Corollary~\ref*{#1}}}
\newrefformat{sec}{\savehyperref{#1}{Section~\ref*{#1}}}
\newrefformat{app}{\savehyperref{#1}{Appendix~\ref*{#1}}}
\newrefformat{assum}{\savehyperref{#1}{Assumption~\ref*{#1}}}
\newrefformat{asm}{\savehyperref{#1}{Assumption~\ref*{#1}}}
\newrefformat{asp}{\savehyperref{#1}{Assumption~\ref*{#1}}}
\newrefformat{ex}{\savehyperref{#1}{Example~\ref*{#1}}}
\newrefformat{fig}{\savehyperref{#1}{Figure~\ref*{#1}}}
\newrefformat{alg}{\savehyperref{#1}{Algorithm~\ref*{#1}}}
\newrefformat{rem}{\savehyperref{#1}{Remark~\ref*{#1}}}
\newrefformat{conj}{\savehyperref{#1}{Conjecture~\ref*{#1}}}
\newrefformat{prop}{\savehyperref{#1}{Proposition~\ref*{#1}}}
\newrefformat{proto}{\savehyperref{#1}{Protocol~\ref*{#1}}}
\newrefformat{prob}{\savehyperref{#1}{Problem~\ref*{#1}}}
\newrefformat{claim}{\savehyperref{#1}{Claim~\ref*{#1}}}
\newrefformat{que}{\savehyperref{#1}{Question~\ref*{#1}}}
\newrefformat{op}{\savehyperref{#1}{Open Problem~\ref*{#1}}}
\newrefformat{fn}{\savehyperref{#1}{Footnote~\ref*{#1}}}

\def \epsilon {\varepsilon}

\newcommand{\ellzo}{\ell^{\text{0-1}}}
\newcommand{\ellsq}{\ell^{\text{sq}}}
\newcommand{\hatellzo}{\hat{\ell}^{\text{0-1}}}

\newcommand{\symdiff}{\mathbin{\Delta}}
\newcommand{\rot}[2]{R(#1,#2)}
\newcommand{\loss}{\ell}

\DeclareMathOperator{\FTPL}{FTPL}
\DeclareMathOperator{\BTPL}{BTPL}

\newcommand{\lossvec}{\boldsymbol{\ell}} 

\title[Toward Simultaneously Optimal Regret in U-Calibration]{Toward Simultaneously Optimal Regret in U-Calibration}
\usepackage{times}
\usepackage{algorithm}
\usepackage{algpseudocode}

\coltauthor{
 \Name{Rafael Frongillo} \Email{raf@colorado.edu}\\
 \addr University of Colorado Boulder
 \AND
 \Name{Haipeng Luo} \Email{haipengl@usc.edu}\\
 \addr University of Southern California and Google Research
 \AND
 \Name{Nishant A. Mehta} \Email{nmehta@uvic.ca}\\
 \addr University of Victoria
 \AND
 \Name{Jon Schneider} \Email{jschnei@google.com}\\
 \addr Google Research
}

\begin{document}

\maketitle

\begin{abstract}
U-calibration studies online forecasting algorithms whose predictions can be consumed by any unknown downstream agent, guaranteeing sublinear regret simultaneously for all proper loss functions. Existing U-calibration algorithms achieve worst-case optimal $O(\sqrt{T})$ regret for every bounded proper loss, but they fail to adapt to easier losses: as we show, even for smooth losses such as squared loss, they incur $\Omega(\sqrt{T})$ regret instead of the optimal $O(\log T)$ regret.

In this work, we show that this limitation is not inherent. Specifically, we design a single forecast algorithm that simultaneously achieves $\tilde O(\sqrt{T})$ regret for every bounded proper loss and $O(\log T)$ regret for every bounded smooth proper loss. More generally, our algorithm also attains logarithmic regret for losses that are smooth relative to the log-barrier, which include several non-Lipschitz examples.
Our approach is based on a novel variant of Follow-the-Perturbed-Leader (FTPL) in which perturbations are applied directly in the prediction space using \emph{self-concordant noise}. 
The resulting analysis also departs substantially from prior FTPL analyses due to the complex nature of this noise and may be of independent interest.
\end{abstract}


\section{Introduction}

In recent years, a growing literature demonstrates the feasibility of ``prediction as a service'' in online learning. In this paradigm, a single algorithm publishes probabilistic predictions with the intent that they will be consumed by an unknown downstream decision maker (who in turn must select an action whose utility depends on the outcome of the event being predicted). The goal of the algorithm is that their predictions should be \emph{trustworthy}: the predictions should be sufficiently accurate so that any downstream decision maker should willingly trust these predictions over any simpler baselines that they may have access to. In particular, these downstream decision makers should always have low regret compared to following any static (``base rate'') prediction. 

More formally, consider an algorithm, that in every round $t = 1, 2, \dots, T$ must produce a probabilistic prediction $P_t \in \Delta_{K}$ of a random outcome $y_t \in \{e_1, \ldots, e_K\}$ taking one of $K$ values. It can be shown that the utility of any downstream agent corresponds to some proper loss 
function $\ell(p, y): \Delta_{K} \times \{e_1, \ldots, e_K\} \rightarrow \mathbb{R}$ (see \pref{sec:preliminaries} for formal definition), and the regret of this downstream agent is given by
\[
\Reg_\ell = \sum_{t=1}^T \ell(P_t, y_t) - \min_{p^\star\in \Delta_K}\sum_{t=1}^T \ell(p^\star, y_t).
\]

The goal of U-Calibration is to produce one sequence of online predictions that guarantees sublinear $\Reg_{\ell}$ for \emph{every} (bounded) proper loss $\ell$. \cite{kleinberg2023u-calibration} showed that this is in fact possible, providing an algorithm guaranteeing $\Reg_{\ell} = O(K\sqrt{T})$, with \cite{luo2024optimal} later improving the dependence on $K$ to the worst-case optimal $O(\sqrt{KT})$. Taken together, these results show the promise of a single prediction service that could obviate the need for each decision maker to run their own learning algorithm specifically tailored to their particular decision problem.

However, these previous results elide a crucial detail: not all downstream agents necessarily must incur $\Omega(\sqrt T)$ regret. For example, for the quadratic loss $\ell(p, y) = (p-y)^2$, it is known that $\Reg_{\ell} = O(\log T)$ is achievable by simply predicting the historical average prediction so far. Such an agent may therefore be unsatisfied with the downstream guarantees of the existing $U$-calibration algorithms, and opt instead to run their own learning algorithm. This naturally leads to the following question: \emph{does there exist an online forecasting algorithm which achieves (asymptotically in $T$) the optimal regret rate for every proper loss $\ell$}?

\subsection{Our Contributions}

We give a positive answer to this question for a wide class of loss functions for which $O(\log T)$ regret rates are known. Specifically, we present an efficient online prediction algorithm (\pref{alg:multiclass}) which achieves the following guarantees.

\begin{theorem}\label{thm:multiclass-new-main} [Restatement of \pref{thm:multiclass-new}]
For any bounded proper loss $\ell$, \pref{alg:multiclass} guarantees $\Reg_\ell = \otil\rbr{K^{5/4}\sqrt{T}}$. Simultaneously, for any $\beta$-smooth and bounded proper loss $\ell$, \pref{alg:multiclass} guarantees $\Reg_\ell = O(\beta\log T + \beta \sqrt{K} \log K)$.
\end{theorem}

Here, a $\beta$-smooth proper loss is a loss that is $\beta$-smooth in its first argument, i.e., that satisfies $\nabla^2_p \ell(p, y) \preceq \beta I$. 
In fact, we additionally show (in \pref{prop:multiclass_relative_smooth}) that \pref{alg:multiclass} attains similar logarithmic regret guarantees for a wider class of loss functions that are 
$\beta$-smooth relative to log barrier (see \pref{sec:preliminaries} for a definition).

In contrast to \pref{thm:multiclass-new-main}, all previously existing algorithms for U-calibration (i.e., those by \cite{kleinberg2023u-calibration} and \cite{luo2024optimal}) incur $\Omega(\sqrt{T})$ regret for smooth losses, even in the first $O(\sqrt{T})$ rounds (\pref{thm:lower_bound}).
Many of these existing algorithms are variants of Follow-The-Perturbed-Leader, which can be interpreted as adding a number of fictional ``perturbation'' rounds at the beginning of the game, and then predicting the historical average from then on. Unfortunately, in order to attain $O(\sqrt{T})$ regret bounds for any bounded loss, one needs to add at least $\Omega(\sqrt{T})$ perturbation rounds at the beginning, which irrevocably ruins the regret guarantee for smooth losses.

Our Algorithm~\ref{alg:multiclass} is also a variant of the Follow-The-Perturbed-Leader algorithm (FTPL), but with the following adaptations:

\begin{itemize}
    \item First, we sidestep the issue mentioned above by applying the perturbation not to the sequence of outcomes, but to the predictions themselves. This can be interpreted as \emph{incremental} noise in outcome space---the magnitude of the noise increases as the game continues, instead of being fixed throughout. This requires a fundamentally different analysis of FTPL than in the previous works in U-calibration; in particular, we directly analyze the stability of an appropriate coupling of the perturbation sequence, as in \cite{kalai2005efficient}.

    \item Second, adding most common forms of noise directly to the predictions $P_t$ has the unfortunate side-effect of potentially causing the predictions to leave the probability simplex (e.g., predicting negative probabilities for some outcomes). For some losses, this issue is easily addressed by projecting the perturbed prediction back to the probability simplex, but it is unclear whether a single projection could weakly improve regret for all proper losses simultaneously.

    Instead, we address this issue by introducing a new form of noise for our perturbations that we term \emph{self-concordant noise} (inspired by the properties of self-concordant barriers) which is guaranteed to never leave the probability simplex. In particular, we choose $p_t$ uniformly from an ellipsoid contained in the probability simplex centered at the empirical average prediction. The prediction-dependence of this noise makes it more complex to directly analyze, and this analysis is the key technical novelty of this paper.  
\end{itemize}

\subsection{Related Work}

\paragraph{Omniprediction and Online Omniprediction} Strongly related to U-calibration is the concept of \emph{omniprediction} \citep{gopalan2022omnipredictors} which is the goal of producing predictions that simultaneously minimize a large class of losses. Several recent works have focused on the goal of providing \emph{online} omniprediction guarantees \citep{okoroafor2025near, garg2024oracle, bechavod2025online}, similar in spirit to the $O(\sqrt{KT})$ online U-calibration guarantees mentioned above. The direct analogue of this paper would involve constructing an online omniprediction algorithm with better regret rates for more amenable losses.

\paragraph{Downstream Swap Regret} Several recent papers have focused on obtaining the stronger guarantee of minimizing \emph{swap regret} for all downstream agents \citep{roth2024forecasting, hu2024calibration, lu2025sample} -- in particular, \cite{hu2024calibration} show it is possible to get $\tilde{O}(\sqrt{T})$ swap regret for all downstream agents in the binary outcome setting. \cite{luo2025simultaneous} show it is possible to get an improved bound of $\tilde{O}(T^{1/3})$ swap regret simultaneously for all proper losses with a smooth univariate form  $p \mapsto \E_{Y \sim p} [ \ell(p, Y) ]$.

We discuss some additional related work in \pref{app:more-related-work}.

\section{Setting and Preliminaries}
\label{sec:preliminaries}

We consider the following fundamental problem of sequential probabilistic predictions.
For each time $t=1, \ldots, T$, a forecaster chooses a potentially random probability distribution $P_t \in \Delta_K$, where $\Delta_K$ is the probability simplex over $K$ possible outcomes.
Simultaneously, an adversary decides an outcome $y_t \in \{e_1, \ldots, e_K\}$ where $e_i$ is the $i$-th standard basis vector in $\fR^{K}$.
At the end of time $t$, outcome $y_t$ is revealed to the forecaster.
For simplicity, we assume that the adversary is oblivious, that is, $y_t$ is independent of the forecaster's previous predictions $P_1, \ldots, P_{t-1}$.
Equivalently, we can think of the adversary picking the outcome sequence $y_1, \ldots, y_T$ ahead of time, knowing the forecaster's algorithm but not their randomness.\footnote{All our results directly generalize to an adaptive adversary (by drawing fresh noise in each time, according to~\citealp[Lemma~12]{hutter05adaptive}); see also \pref{fn:adaptive_adversary} for more details.}

For a loss function $\ell: \Delta_K \times\{e_1, \ldots, e_K\} \rightarrow [-1,1]$, the forecaster's expected regret with respect to $\ell$ is defined as
\[
\Reg_\ell = \E\left[\sum_{t=1}^T \ell(P_t, y_t)\right] - \inf_{p^\star\in \Delta_K}\sum_{t=1}^T \ell(p^\star, y_t),
\]
where the expectation is over the internal randomness of the forecaster.
In words, regret compares the forecaster's total loss to that of the best fixed prediction in hindsight.

\paragraph{Proper losses}
Throughout the paper, we focus on losses that are proper, meaning that for every distribution $p \in \Delta_K$, the prediction $p$ minimizes the expected loss when the outcome is actually drawn from $p$; that is, $p \in \argmin_{p' \in \Delta_K} \E_{Y \sim p}[\ell(p', Y)]$.
Note that in this case, by definition, the optimal prediction $p^\star$ in the regret definition is simply $\pbar_T$, where we use $\pbar_t = \frac{1}{t}\sum_{\tau=1}^t y_\tau$ to denote the empirical average of the outcomes up to time $t$.

We let $\calL$ denote the set of all such bounded proper losses.
\citet{kleinberg2023u-calibration} and \citet{luo2024optimal} show that a single algorithm can make predictions with guarantee $\Reg_\ell=\order(\sqrt{KT})$ simultaneously for all proper losses $\ell\in\calL$, which is worst-case optimal.
However, \cite{luo2024optimal} also identify a large subclass of proper losses where $\order(\log T)$ regret is possible.
Our goal is to design algorithms that ensures not only $\Reg_\ell = \order(\sqrt{KT})$ for all $\ell\in\calL$, but also at the same time a better regret bound for a large subclass of these losses.

\paragraph{Smooth proper losses}
In particular, we consider $\calS_\beta \subseteq \calL$, the subclass of all $\beta$-smooth 
proper losses with loss range $[-1, 1]$.
Here, a loss $\ell$ is $\beta$-smooth (in its first argument) if for any $p, q\in \Delta_K$ and $y$, we have $\ell(q, y) - \ell(p,y) \leq \inner{\nabla_p\ell(p,y), q-p} + \frac{\beta}{2}\|q-p\|_2^2$ (where $\nabla_p\ell(p,y)$ is the gradient of $\ell$ as a function of $p$).
For example, the squared loss $\ell(p,y) = \frac{1}{2}\|p-y\|_2^2$ is $1$-smooth (and proper).
Since our domain is bounded, any $\ell \in \calS_\beta$ also has to be $O(\beta)$-Lipschitz, which means that, according to~\citet{luo2024optimal}, a simple Follow-the-Leader (FTL) strategy (that is, predict $P_t = \pbar_{t-1}$) enjoys $O(\beta\log T)$ regret.
On the other hand, it is well known that there exist proper losses where FTL suffers linear regret; see  e.g.~\citet[Theorem~6]{luo2024optimal}.

In fact, our results hold beyond this particular smooth class.
To illustrate the idea, we additionally consider the class of proper losses that are $\beta$-smooth \textit{relative to log-barrier}, that is, 
\[
\nabla^2_p \ell(p, y) \preceq \beta \cdot \nabla^2 \left(-\sum_{i=1}^K \log p_i\right) = \beta\cdot\diag\left(\frac{1}{p_1^2}, \ldots, \frac{1}{p_K^2}\right)
\] 
for any $p$ and $y$, where $\diag(a_1, \ldots, a_K)$ denotes the $K$ by $K$ diagonal matrix with diagonal values $a_1, \ldots, a_K$ and $A \preceq B$ means $B-A$ is a positive-semidefinite matrix.
We use $\calS_\beta^{\log}$ to denote this class and note the relation $\calS_\beta \subset \calS_\beta^{\log} \subset \calL$.
As an example, the following loss considered by \citet{luo2024optimal} is neither Lipschitz nor smooth, but it belongs to $\calS_\beta^{\log}$ for some value of $\beta$:
\begin{equation}\label{eq:log_smooth_ex}
\ell(p,y) = C\left((\alpha-1)\sum_{i=1}^K p_i^\alpha - \alpha \sum_{i=1}^K p_i^{\alpha-1}y_i\right),    
\end{equation}
where $\alpha \in (1,2)$ and $C>0$ is a rescaling constant so that $\ell(p,y)\in[-1,1]$ for all $p$ and $y$.

\section{Warm-up with binary outcomes}
\label{sec:warm-up}

Let us begin with the simpler setting of sequential binary prediction, where an algorithm submits a random variable $P_t\in[0,1]$ in each round, after which an adversary chooses an outcome $y_t \in \{0,1\}$.
Perhaps the two most natural losses to consider are 0-1 loss and squared loss.
In the binary setting, squared loss can be written more simply as $\ellsq(p,y) = (p-y)^2$.

For 0-1 loss, it is instructive to begin with the more familiar action setting.
For a set of actions $\A$, one can define a loss $\hat\ell : \A \times \{0,1\} \to \R$.
For example, when $\A = \{0,1\}$, 0-1 loss is $\hatellzo(a,y) = \ones\{a \neq y\}$.
We can easily convert any such loss to a proper loss by encoding a Bayes action $a_p \in \argmin_{a\in\A} \E_{Y\sim p} [ \hat\ell(a,Y) ]$ for each $p \in \Delta_K$, and defining $\ell(p,y) = \hat\ell(a_p,y)$.
For 0-1 loss, that gives us the proper loss
$\ellzo(p,y) = \hatellzo(\ones\{p > 1/2\},y) = \ones\{\ones\{p > 1/2\} \neq y\}$.

Every ``V-shaped'' loss~\citep{kleinberg2023u-calibration} can also be expressed concisely via a binary action $a \in \{0,1\}$.
For $\gamma \in [0,1]$, let $\hat\ell_\gamma$ be the cost-sensitive loss
\[
  \hat\ell_\gamma(0,0)=\hat\ell_{\gamma}(1,1)=0,\quad \hat\ell_{\gamma}(0,1)=1,\quad \hat\ell_{\gamma}(1,0) = \gamma~.
\]
Then $\ell_\gamma(p,y) = \hat\ell_\gamma(\ones\{p > p_\gamma\},y)$ is the corresponding proper loss, where $p_\gamma = \gamma/(1+\gamma) \leq 1/2$.
(The $p_\gamma \geq 1/2$ case follows symmetrically, where now $\hat\ell_\gamma(0,1)=\gamma$ and $\hat\ell_\gamma(1,0) = 1$.)

The main question we address in this section is the following:
Does there exist an algorithm to choose $P_t$ so that $\Reg_{\ellsq} = O(\log T)$ and $\Reg_{\ellzo} = O(\sqrt T)$ or even $\Reg_{\ell_\gamma} = O(\sqrt T)$ for all $\gamma \in [0,1]$?
In what follows, 
we first show that previous U-Calibration algorithm does not suffice, and then give a simple algorithm that does achieve both bounds simultaneously.
In fact, the same algorithm satisfies $\Reg_\ell = O(\sqrt T)$
for all $\ell\in\L$.

\subsection[$\Omega(\sqrt{T})$ lower bound for existing U-Calibration algorithms]{$\boldsymbol{\Omega(\sqrt{T})}$ lower bound for existing U-Calibration algorithms}

Let us first see why the \textsc{ForecastHedge} algorithm of \cite{kleinberg2023u-calibration} does not suffice.
In particular, we will show that it can suffer $\Theta(\sqrt T)$ expected regret under squared loss, essentially because the variance of the predictions $P_t$ is too high.

\begin{theorem}\label{thm:lower_bound}
For all sufficiently large $T$, there exists an adversarial sequence of (binary) outcomes where \textsc{ForecastHedge} algorithm \citep{kleinberg2023u-calibration} incurs $\Omega(\sqrt{T})$ expected regret with respect to the 
squared loss $\ellsq(p, y) = (p-y)^2$.

\end{theorem}
\begin{proof}
  Consider the constant sequence $y_t = 0$.
  For this sequence, the best fixed prediction in hindsight is simply $p^{*} = 0$, which incurs a total loss of $0$.
  Thus, we have
  \begin{align*}
    \Reg_\ellsq = \sum_{t=1}^T \E [ \ellsq(P_t,0) ] = \sum_{t=1}^T \E[P_t^2] \geq \sum_{t=1}^T \Pr[P_t = 1]~.
  \end{align*}
  Recall that \textsc{ForecastHedge} samples $P_t$ according to a distribution with cumulative distribution function
  \[
    \Pr[P_t \le p] = S\!\left(\frac{t-1}{\sqrt{T}}(p - \bar{p}_{t-1})\right),
    \quad p \in [0,1),
  \]
  where $S(x) = \frac{1}{1 + e^{-2x}}$.
  For all $2 \leq t \leq \lfloor \sqrt{T}\rfloor+1$ we have $\bar p_{t-1}=0$ and
  \[
    \Pr[P_t = 1]
    = 1 - S\left(\frac{t-1}{\sqrt{T}}\right)
    \geq 1 - S(1)
    = \frac{1}{1+e^2}~.
  \]
  Thus \quad
  $\displaystyle
    \Reg_\ellsq
    \geq \sum_{t=1}^T \Pr[P_t = 1]
    \geq \sum_{t=2}^{\lfloor \sqrt{T}\rfloor+1} \Pr[P_t = 1]
    \geq \lfloor \sqrt{T}\rfloor \cdot \frac{1}{1+e^2}
    = \Omega(\sqrt{T})~.
  $
\end{proof}
Similarly, one can show that the \textsc{ForecastFTPL} algorithm of ~\cite{kleinberg2023u-calibration} (for the multiclass setting) as well as the algorithm of~\citet{luo2024optimal} both also suffer the same issue (details omitted).

\subsection{Achieving simultaneous optimal regret}

As FTL does achieve $O(\log T)$ expected regret for squared loss, a natural approach is therefore to choose $P_t = \bar p_{t-1} + Z_t$, where $\pbar_{t}=\frac{1}{t}\sum_{s=1}^t y_s$ and the $Z_t$'s are independent zero-mean random variables with sufficiently low variance $\sigma^2_t$.
The following lemma shows that keeping the total variance to $O(\log T)$ suffices to maintain the $O(\log T)$ regret against squared loss.

\begin{lemma}\label{lem:squared-loss-decomp}
  Define $\mu_t := \E [ P_t ]$, $\sigma^2_t := \Var(P_t)$. 
  Then $\E [ (P_t-y_t)^2 ] = (\mu_t-y_t)^2 + \sigma^2_t$, and thus 
   
  $
    \Reg_{\ellsq}
    = \Reg_{\ellsq}(\{\mu_t\}_t) + \sum_{t=1}^T \sigma^2_t~,
    $
    where $\Reg_{\ellsq}(\{\mu_t\}_t) = \sum_{t=1}^T (\mu_t-y_t)^2 - \sum_{t=1}^T(\pbar_T-y_t)^2$.
\end{lemma}

As we will see, the key to keeping the variance low while also achieving $O(\sqrt T)$ expected regret for all proper losses is to use a time-varying learning rate such as $\eta_t = 1/\sqrt t$.
To motivate this choice, consider running FTPL on 0-1 loss, where the perturbations happen in action space.
That is, letting $c_{t,y} := |\{s \leq t : y_s = y\}|$ be the counts of each outcome $y\in\{0,1\}$ so far, FTPL takes
\begin{align}
  \label{eq:ftpl}
  a_t = \argmax_{a\in\{0,1\}} \; c_{t-1,a} + \frac 1 \eta_t W_a~,
\end{align}
where $W_0,W_1$ are i.i.d.\ copies of some noise random variable $W$.
Normalize $W$ so that $\Var(W) = 1$ for simplicity.
For example, taking $W$ to have a Gumbel distribution gives the well-known Hedge algorithm, which for the choice $\eta_t = 1/\sqrt t$ achieves $O(\sqrt T)$ regret with respect to the best fixed action in hindsight \citep[Section 2.3]{cesabianchi2006prediction}.

It turns out that we can exactly implement this algorithm by taking $P_t = \bar p_{t-1} + Z_t$ where $Z_t = (1/2\eta_t)(W_1 - W_0)/(t-1)$, ignoring for the moment that $Z_t$ may cause $P_t$ to leave the interval $[0,1]$.
Setting $\eta_t = 1 / \sqrt t$, we have $\sigma^2_t := \Var(Z_t) = \frac t {(t-1)^2}$, giving the desired total variance of $\sum_{t=1}^T \sigma^2_t = \sum_{t=1}^T \tfrac t {(t-1)^2} = \Theta(\log T)$.
This choice of $P_t$ exactly implements FTPL for 0-1 loss from \pref{eq:ftpl}:
\begin{align*}
  P_t \leq 1/2
  &\iff \bar p_{t-1} + \frac{1}{2\eta_t}\frac{W_1 - W_0}{t-1} \leq 1/2
  \\
  &\iff \frac{c_{t-1,1}}{t-1} + \frac{1}{2\eta_t}\frac{W_1 - W_0}{t-1} \leq 1/2
  \\
  &\iff 2 c_{t-1,1} + \frac{1}{\eta_t}(W_1 - W_0) \leq c_{t-1,0} + c_{t-1,1}
  \\
  &\iff c_{t-1,1} + \frac{1}{\eta_t} W_1 \leq c_{t-1,0} + \frac 1 \eta_t W_0~.
\end{align*}

In fact, this reduction to the binary-action FTPL algorithm extends to all V-shaped losses $\ell_\gamma$ defined above.
To see this, note that
using the notation above, the cumulative losses under $\hat\ell_\gamma$ of action $a\in\{0,1\}$ are
\begin{align*}
  L_{t,0}
  &:= \sum_{s=1}^t \hat\ell_\gamma(0,y_t)
    = \sum_{s=1}^t \ones\{y_t=1\}
    = c_{t,1}~,
  \\
  L_{t,1}
  &:= \sum_{s=1}^t \hat\ell_\gamma(1,y_t)
    = \sum_{s=1}^t \gamma \ones\{y_t=0\}
    = \gamma c_{t,0}~,
\end{align*}
and FTPL chooses
\begin{align}
  \label{eq:ftpl-V-shaped}
  a_t = \argmax_{a\in\{0,1\}} \; L_{t-1,a} + \frac 1 {\eta_t^{(\gamma)}} W_a~.
\end{align}
Just as with 0-1 loss, we have
\begin{align*}
  P_t \leq p_\gamma
  &\iff L_{t-1,0} + \frac{1+\gamma}{2\eta_t} W_0 \leq L_{t-1,1} + \frac{1+\gamma}{2\eta_t} W_1~.
\end{align*}
Thus, V-shaped losses using these choices of $P_t$ are exactly running FTPL with learning rate $\eta_t^{(\gamma)} = \tfrac 2 {1+\gamma} \eta_t$.
The choice $\gamma=1$ recovers 0-1 loss above with $\eta_t^{(\gamma)} = \eta_t$.

Finally, let us contend with the fact that we must require $P_t \in [0,1]$. 
More precisely, we will take $\tilde P_t = \bar p_{t-1} + Z_t$ and $P_t = \mathrm{clip}_{[0,1]}(\tilde P_t)$, where $\mathrm{clip}_{[0,1]}(x) = \min\{ \max\{x, 0\}, 1 \}$.
For the binary action losses, one easily checks that moving from $\tilde P_t$ to $P_t = \mathrm{clip}_{[0,1]}(\tilde P_t)$ does not change either decision.
For squared loss, one similarly observes that clipping 
cannot increase the expected regret.
This is because the comparator term is unchanged, and the loss suffered by the algorithm can only decrease: for all $\tilde p \in \R$, we have $(\mathrm{clip}_{[0,1]}(\tilde p) - y)^2 \leq (\tilde p-y)^2$ for all $y\in\{0,1\}$.
We therefore still have $O(\log T)$ overall expected regret for squared loss by
\pref{lem:squared-loss-decomp}.

In summary, in the binary setting, we have given an algorithm which achieves $O(\log T)$ expected regret for squared loss and $O(\sqrt T)$ expected regret for all V-shaped losses.
Specifically, we take $P_t = \mathrm{clip}_{[0,1]}(\bar p_{t-1} + Z_t)$ where $Z_t = (1/2\eta_t)(W_1 - W_0)/(t-1)$, $\eta_t = 1 / \sqrt t$, and $W_a$ are i.i.d.\ Gumbel random variables with variance 1.
Additionally, since \citet{kleinberg2023u-calibration} show that $\sup_{\ell\in\calL} \Reg_\ell \leq 2\sup_{\gamma \in [0,1]} \Reg_{\ell_\gamma}$, 
this algorithm also achieves $O(\sqrt T)$ regret for any $\ell \in \L$.

\section{Multiclass setting}
\label{sec:multiclass}

Similar to the case of binary prediction, our algorithm for the multi-class setting adds noise in the prediction space $\Delta_K$. A natural approach to extend our algorithm for the binary setting would be to try to perturb individual coordinates using independent Gumbel noise (or, e.g., Gaussian noise); however, it is unclear in general how to deal with perturbed probability vectors that no longer belong to the simplex. 
For example, consider $K \geq 3$ and a perturbed probability vector $p$ where all coordinates are in $[0, 1]$ but $p$ is not in the simplex.  What is the ``right'' projection onto the simplex? 
\reversemarginpar 
In order to assess the goodness of a projection, one would first need to extend the first argument of the loss function beyond the simplex. 
If the loss could be suitably extended, one would still need 
a single projection (for all losses) that never increases the loss.

Because of these challenges, we instead take care to keep the prediction within $\Delta_K$, via a novel (to our knowledge) perturbation which we dub \emph{self-concordant noise} (see \pref{app:more-related-work} for more discussion on how this idea is connected to self-concordant barriers).

We use $\fB^K_2$ to denote the $K$-dimensional $\ell_2$ unit ball and $\sigma\fB^K_2$ for some $\sigma > 0$ to denote $\fB^K_2$ scaled by $\sigma$. 
To perturb a given probability vector $p$ with self-concordant noise (at scale $\sigma$), we:
\begin{enumerate}
\item Draw 
$S$ uniformly at random from $\sigma \fB^K_2$ intersected with the subspace orthogonal to $p$.
\item Coordinate-wise scale $S$ by $p$, giving self-concordant noise $Z = \diag(p) S$.
\end{enumerate}
It is easy to verify that for any $p$ in the simplex $\Delta_K$ (including the boundary), $p + Z$ also belongs to the simplex, as long as $\sigma \leq 1$. \pref{alg:multiclass} formally presents our algorithm that uses such self-concordant noise and sets $\sigma$ to be of order $1/\sqrt{T}$.
Our main result is the following simultaneous regret guarantees for \pref{alg:multiclass}.\footnote{We emphasize again that our results hold even under an adaptive adversary, even though our analysis assumes an oblivious adversary. This is because the loss of the algorithm for round $t$ depends solely on $y_1, \ldots, y_t$ and $S_t$, so the extra knowledge of $S_1, \ldots, S_{t-1}$ (which are independent of $S_t$ given $y_1, \ldots, y_t$) does not make an adaptive adversary any more powerful than an oblivious adversary.\label{fn:adaptive_adversary}}

\SetAlgoNoEnd 

\begin{algorithm}[t]
   \caption{Simultaneously Optimal Multiclass U-Calibration}
   \label{alg:multiclass}
   {\bfseries Initialize:} $\sigma = K^{3/4} / \sqrt{T}$, $\pbar_{0}$ is the uniform distribution.
   
   \For{$t=1, \ldots, T$}{

   Uniformly at random sample $S_t$ from $\sigma\fB_2^K \cap \{s \in \fR^K: \inner{s, \pbar_{t-1}} = 0\}$.

   Predict $P_t = \pbar_{t-1} + Z_t$ where $Z_t = \diag(\pbar_{t-1})S_t$.

   Observe label $y_t$ and update the empirical average $\pbar_{t} = \frac{1}{t}\sum_{\tau=1}^{t} y_\tau$.

   }
\end{algorithm}

\begin{theorem}\label{thm:multiclass-new}
\pref{alg:multiclass} achieves $\Reg_\ell = \otil\rbr{K^{5/4}\sqrt{T}}$ for any bounded proper loss $\ell \in \calL$ and simultaneously $\Reg_\ell = O(\beta\log T + \beta \sqrt{K} \log K)$ for any smooth and bounded proper loss $\ell\in\calS_\beta$.
\end{theorem}

\paragraph{Analysis for $\boldsymbol{O(\log T)}$ regret}
We start by proving the guarantee for smooth losses and restate the result below with an intermediate bound that depends on an arbitrary $\sigma \leq 1$.
The proof directly generalizes the idea from \pref{sec:warm-up} and argues that our algorithm is not too far away from FTL.

\begin{proposition}\label{prop:multiclass_smooth}
For any $\beta>0$ and any loss function $\ell\in \calS_\beta$, \pref{alg:multiclass} achieves 
$\Reg_\ell \leq O(\beta\log T + \beta\sigma^2 T(\log K)/K) = O(\beta\log T + \beta \sqrt{K} \log K)$.  
\end{proposition}
\begin{proof}
First, note that since $\ell$ is defined over $\Delta_K$, a space with $O(1)$ diameter in $\ell_2$ norm, the fact that it is $\beta$-smooth also implies that it is $O(\beta)$-Lipschitz.
Therefore, according to~\citet{luo2024optimal}, the regret of the FTL strategy (that is, predict $\pbar_{t-1}$ at time $t$) is $O(\beta\log T)$.
It thus remains to analyze the difference between the loss of \pref{alg:multiclass} and that of FTL, i.e., $\E\sbr{\sum_{t=1}^T \ell(P_t, y_t) - \ell(\pbar_{t-1}, y_t)}$.
To do so, we plug in the definition of $P_t=\pbar_{t-1}+Z_t$ and use the smoothness property:
\begin{align*}
\E_{Z_t}\sbr{\ell(P_t, y_t) - \ell(\pbar_{t-1}, y_t)}
&\leq \E_{Z_t}\sbr{\inner{\nabla\ell_p(\pbar_{t-1}, y_t), Z_t}} + \frac{1}{2}\beta \E_{Z_t}\sbr{\|Z_t\|_2^2} \\
&=  \frac{1}{2}\beta \E_{S_t}\sbr{\|\diag(\pbar_{t-1})S_t\|_2^2} && (\E_{Z_t}\sbr{Z_t} = 0) \\
&\leq \frac{1}{2}\beta \E_{S_t}\sbr{\|S_t\|_\infty^2} .
\end{align*}
\pref{lem:subgaussian-sphere} in \pref{app:beta-smooth} establishes that $S_t$ is a subgaussian random vector and gives the bound
$
\E_{S_t}\sbr{\|S_t\|_\infty^2} 
\leq \frac{8 \log K}{K-1}$. 
Hence, we have
\begin{align*}
\E_{Z_t}\sbr{\ell(P_t, y_t) - \ell(\pbar_{t-1}, y_t)}
\leq \frac{8 \beta \sigma^2 \log K}{K-1} 
= O\rbr{\frac{\beta \sigma^2 \log K}{K}} .
\end{align*}
Summing over $t$ finishes the proof.
\end{proof}

In fact, by examining the proof more carefully, one can see that the statement that our algorithm is close to FTL can be further extended to losses that are not necessarily smooth but are smooth relative to log-barrier, as show in the following proposition.

\begin{proposition}\label{prop:multiclass_relative_smooth}
For any $\beta>0$ and any loss function $\ell\in \calS_\beta^{\log}$, the total loss of \pref{alg:multiclass} compared to that of the Follow-The-Leader strategy is bounded as 
$\E\sbr{\sum_{t=1}^T \ell(P_t, y_t) - \ell(\pbar_{t-1}, y_t)}= O(\beta \sigma^2 T) = O(\beta K^{3/2})$. 
\end{proposition}
\begin{proof}
We plug in the definition of $P_t=\pbar_{t-1}+Z_t$ and apply second-order Taylor expansion:
\begin{align*}
&\E_{Z_t}\sbr{\ell(P_t, y_t) - \ell(\pbar_{t-1}, y_t)} \\
&= \E_{Z_t}\sbr{\inner{\nabla_p\ell(\pbar_{t-1}, y_t), Z_t}} + \frac{1}{2}\E_{Z_t}\sbr{Z_t^\top \nabla^2 \ell_p(\xi, y) Z_t} \tag{for some $\xi$ between $P_t$ and $\pbar_{t-1}$}\\
&\leq  \frac{1}{2}\beta \E_{Z_t}\sbr{Z_t^\top \diag\rbr{\frac{1}{\xi_1^2}, \ldots, \frac{1}{\xi_K^2}} Z_t^\top} \tag{$\E_{Z_t}\sbr{Z_t} = 0$ and relative smoothness} \\
&\leq  \frac{1}{2}\beta \E_{S_t}\sbr{S_t^\top \diag\rbr{\frac{\pbar_{t-1,1}^2}{\xi_1^2}, \ldots, \frac{\pbar_{t-1,K}^2}{\xi_K^2}} S_t^\top} \tag{$Z_t = \diag(\pbar_{t-1})S_t$}. 
\end{align*}
Now, for each $i$, note that
\[
\frac{\pbar_{t-1,i}}{\xi_i} \leq 
\frac{\pbar_{t-1,i}}{\min\cbr{P_{t,i},\pbar_{t-1,i}}}
=\frac{\pbar_{t-1,i}}{\min\cbr{\pbar_{t-1,i}(1+S_{t,i}),\pbar_{t-1,i}}}
\leq \frac{1}{1-\sigma} = O(1).
\]
Therefore, continuing from the earlier derivation,
we have
\[
\E_{Z_t}\sbr{\ell(P_t, y_t) - \ell(\pbar_{t-1}, y_t)}
= O(\beta \E_{S_t}\sbr{\|S_t\|_2^2}) = O(\beta\sigma^2).
\]
Summing over $t$ finishes the proof.
\end{proof}

As mentioned, \citet{luo2024optimal} show that Follow-The-Leader achieves $O(\log T)$ regret for Lipschitz proper losses, which implies that \pref{alg:multiclass} achieves $O(\log T)$ regret for all losses that are Lipschitz and smooth relative to log-barrier.
Moreover, \citet{luo2024optimal} also identify another broad class of proper losses that are not necessarily Lipschitz and for which Follow-The-Leader still achieves $O(\log T)$ regret.
One such example is \pref{eq:log_smooth_ex}, which, as mentioned, belongs to $\calS_\beta^{\log}$ for some $\beta$,
meaning that our \pref{alg:multiclass} also achieves $O(\log T)$ regret in this case.

\paragraph{Analysis for $\boldsymbol{O(\sqrt{T})}$ regret}
We next 
prove the $O(\sqrt{T})$ regret bound of~\pref{thm:multiclass-new-main}, again starting with a restatement of the result that includes an intermediate bound in terms of an arbitrary $\sigma \leq 1/2$.

\begin{proposition}\label{prop:multiclass_proper-new}
For any loss function $\ell\in \calL$, \pref{alg:multiclass} achieves 
\begin{align*}
\textstyle 
\Reg_\ell 
= \otil\rbr{ \frac{K^{2}}{\sigma} + T \sqrt{K} \sigma }
= \otil\rbr{ K^{5/4}\sqrt{T} } .
\end{align*}
\end{proposition}

Unlike the $O(\log T)$ results, due to the complex nature of our self-concordant noise, 
the proof for \pref{prop:multiclass_proper-new} no longer follows the same idea of \pref{sec:warm-up} that establishes equivalence to existing no-regret algorithms in the action space.
Instead, we propose to directly analyze our algorithm via the typical FTPL analysis,
which decomposes the regret into two terms: the first term is the loss/regret difference between FTPL and Be-the-Perturbed-Leader (BTPL), also known as the stability term,
and the second term is the regret of BTPL.
While the decomposition is standard,
bounding each of these two terms requires significantly new ideas.

Specifically,
the following lemma bounds the stability term. Establishing this lemma is a major undertaking, so we sketch a proof in \pref{sec:stability}. A full proof can be found in  \pref{app:stability-new}.

\begin{lemma}[Stability Term] \label{lem:ellipsoid-stability-new}
For any loss function $\loss \in \calL$, under \pref{alg:multiclass}, it holds that
\begin{align*}
\textstyle 
\E\sbr{\sum_{t=1}^T \bigl( \loss(P_t, y_t) - \loss(P_{t+1}, y_t) \bigr)} 
= \otil\left( \frac{K^{2}}{\sigma} \right) .
\end{align*}
\end{lemma}

This bound matches our intuition that the larger the variance of the noise, the more stable the algorithm is.
On the other hand, the next lemma bounds the regret of BTPL (that is, an imaginary algorithm that plays $P_{t+1}$ at time $t$); the bound is increasing linearly in the noise level $\sigma$.
Combining both lemmas proves \pref{prop:multiclass_proper-new}.

\begin{lemma}[Regret of BTPL]\label{lem:BTPL_regret}
For any loss function $\ell\in \calL$, \pref{alg:multiclass} ensures that 
\begin{align*}
\textstyle 
\sum_{t=1}^T \bigl( \loss(P_{t+1}, y_t) - \loss(\pbar_{T}, y_t) \bigr)
= O\rbr{T \sqrt{K} \sigma}.
\end{align*}
\end{lemma}

Note that this result holds for all realization of the noise, instead of in expectation only.
We sketch the proof below and defer the full proof to \pref{app:btpl}. \\

\begin{proof}(Sketch, of \pref{lem:BTPL_regret})
First, we show that by affinely extending the loss function in its second argument, 
BTPL can be viewed as Be-The-Leader (BTL) using a sequence of perturbed outcomes $y_t' = y_t + tZ_{t+1}-(t-1)Z_t$.
Applying the standard BTL lemma (e.g., \citealt[Lemma~3.1]{cesabianchi2006prediction}) then shows that the regret of BTPL is of order $\sum_{t=1}^T \|t Z_{t+1} - (t-1) Z_t\|_1$.

The main work is to bound $\|Z_{t+1} - Z_t\|_1$, the difficult part of which is bounding $\|S_{t+1} - S_t\|_2$. Recall that for any $t$, the random variable $S_t$ is uniformly distributed on the intersection $\sigma \fB_2^K$ with the subspace 
$\{s \in \fR^K \colon \langle s, \pbar_{t-1} \rangle = 0\}$. 
Taking the standard Riemannian metric on the unit sphere in $\R^K$, 
there is a unique rotation matrix $R_t$ that sends $\bar{p}_{t-1} / \|\bar{p}_{t-1}\|$ to $\bar{p}_t / \|\bar{p}_t\|$ by traveling along the geodesic between these vectors. 
We observe that $R_t S_t$ and $S_{t+1}$ have the same law. Thus, as we assume an oblivious adversary, we can view the sequence $S_1, S_2, S_3, \ldots$ as being initialized at some $S_1$, with each successive iterate obtained via $S_{t+1} = R_t S_t$. 
This view of the sequence $\{S_t\}_t$ provides a suitable coupling which, after some basic manipulations, gives us good control on $\|S_{t+1} - S_t\|_2$ and allows us to show that $\|Z_{t+1} - Z_t\|_1 = O\rbr{\sqrt{K} \sigma / t}$.
\end{proof}

\section{Analysis of the stability term} \label{sec:stability}

In this section, we sketch our novel analysis for our bound on the stability term (\pref{lem:ellipsoid-stability-new}). 
Our starting point is the following fact (shown in \pref{lem:tv-stability} in \pref{app:stability-new}): letting $F_t$ be the law of random variable $P_t$ and $\|\cdot\|_{\mathrm{TV}}$ be the total variation distance, $\E\sbr{\sum_{t=1}^T \ell(P_t, y_t) - \ell(P_{t+1}, y_t)}$ is at most 
$ 2\sum_{t=1}^T  \|F_{t+1} - F_t\|_{\mathrm{TV}}$.

Now, for any $p$ in the relative interior of $\Delta_K$, define the ellipsoid 
\begin{align*}
\textstyle 
E(p) = 
\left\{ x \in \Delta_K : 
        \sum_{i=1}^K \left( \frac{x_i}{p_i} - 1 \right)^2 \leq \sigma^2
\right\} .
\end{align*}
It is not hard to verify that $P_t$ is drawn from the uniform distribution over $E(\bar{p}_{t-1})$ (when $\bar{p}_{t-1}$ is in the relative interior). 
Therefore, if we can suitably bound the total variation (TV) distance between the uniform distribution on $E(\bar{p}_{t-1})$ and $E(\bar{p}_t)$ for ``most'' rounds (where ``most'' is $T - O(\sqrt{T})$
with respect to $T$), then the stability term will be suitably controlled. The next result bounds this TV distance. 
Let $c_{t, i}$ be the number of times outcome $e_i$ has occurred by the end of round $t$. 

\begin{theorem} \label{thm:ellipsoid-tv-new-main-text}
Assume that $\{p, p'\} = \{\bar{p}_{t-1}, \bar{p}_t\}$, with both $p$ and $p'$ being in the relative interior of $\Delta_K$. Let $d(p, p')$ be the TV distance between the uniform distribution over $E(p)$ and the uniform distribution over $E(p')$. 
If 
$\sigma \leq \frac{1}{2}$, \, 
$t-1 \geq \frac{54 K^{3/2}}{\sigma}$, \, 
and $c_{t-1,i_t} \geq \frac{57 K}{\sigma}$, 
then
\begin{align*}
\textstyle 
d(p, p') 
= O\rbr{
          \frac{K^{3/2}}{\sigma (t-1)}
          + \frac{K}{\sigma c_{t-1,i_t}}
  } .
\end{align*}
\end{theorem}

Our proof sketch of \pref{thm:ellipsoid-tv-new-main-text} is somewhat long, so we first sketch a proof of  of \pref{lem:ellipsoid-stability-new}.

\begin{proof}(of \pref{lem:ellipsoid-stability-new})
Our starting point is \pref{lem:tv-stability}. We will use \pref{thm:ellipsoid-tv-new-main-text} to bound the individual total variation terms $\|F_{t+1} - F_t\|_{\mathrm{TV}}$ for most rounds $t$. To apply \pref{thm:ellipsoid-tv-new-main-text}, we need to ensure that in most rounds, 
both $t$ and $c_{t-1,i_t}$ are suitably large. 

There are at most $O\rbr{\frac{K^{3/2}}{\sigma}}$ rounds where the condition 
$t-1 \geq \frac{54 K^{3/2}}{\sigma}$ 
fails to hold, and there are most $O\rbr{\frac{K^2}{\sigma}}$ rounds where the condition 
$c_{t-1,i_t} \geq \frac{57 K}{\sigma}$ 
fails to hold; 
for all such rounds, we trivially bound $\|F_{t+1} - F_t\|_{\mathrm{TV}}$ by $1$,
while for the remaining rounds, denoted by $\calT$, we apply \pref{thm:ellipsoid-tv-new-main-text} to bound $\|F_{t+1} - F_t\|_{\mathrm{TV}}$.\footnote{Note that $\bar{p}_{t-1}$ and $\bar{p}_t$ might not be in the relative interior of $\Delta_K$; however, they must be in the relative interior of $\Delta_{k}$ for some $k \leq K$.
Therefore, we can still apply \pref{thm:ellipsoid-tv-new-main-text} in the space of $\Delta_{k}$ where the conditions 
$t-1 \geq \frac{54 k^{3/2}}{\sigma}$ 
and 
$c_{t-1,i_t} \geq \frac{57 k}{\sigma}$ 
hold and the conclusion 
$d(p, p')
 = \O\rbr{
       \frac{k^{3/2}}{\sigma t} 
       + \frac{k}{\sigma c_{t-1,i_t}}
   } 
 = \O\rbr{
       \frac{K^{3/2}}{\sigma t} 
       + \frac{K}{\sigma c_{t-1,i_t}}
   }$
is 
the same.}
Together, this leads to $\E\sbr{\sum_{t=1}^T \ell(P_t, y_t) - \ell(P_{t+1}, y_t)}$ being at most 
$2(T- |\calT|) + \E\sbr{\sum_{t \in \calT} \ell(P_t, y_t) - \ell(P_{t+1}, y_t)}$, which is
\begin{align*}
O\rbr{ \frac{K^2}{\sigma} + \frac{K^{3/2}}{\sigma} \sum_{t \in \calT} \frac{1}{t} + \frac{K}{\sigma} \sum_{t \in \calT} \frac{1}{c_{t-1,i_t}} } 
= O\rbr{ \frac{K^2 + K^{3/2} \log T + K^2 \log \frac{T}{K}}{\sigma} }
= \otil\rbr{ \frac{K^2}{\sigma} } 
.
\end{align*}
\end{proof}

We now sketch a proof of \pref{thm:ellipsoid-tv-new-main-text}.
\begin{proof}(Sketch, of \pref{thm:ellipsoid-tv-new-main-text})
Let $V(A)$ be the $(K-1)$-dimensional volume of a set $A$. We abuse notation and let $V(p)$ be shorthand for $V(E(p))$ (likewise for $p'$). Assuming $V(p) \leq V(p')$ without loss of generality. 
First, we show that twice the total variation distance $d(p, p')$ is at most
\begin{align*}    
\textstyle 
2 d(p, p') 
\leq \frac{|V(p') - V(p)|}{\max\{V(p), V(p')\}} 
     + \frac{V(E(p) \symdiff E(p'))}{V(p)} ,
\end{align*}
i.e., the sum of the relative volume difference and the relative volume of the symmetric difference.

The relative volume difference is the easier term to control. We can show that 
\begin{align*}
V(p) 
= C_{K-1} \cdot \sigma^{K-1}  
  \left( \prod_{i=1}^K p_i \right) 
  \left( \frac{1}{K} \sum_{i=1}^K \frac{1}{p_i^2} \right)^{1/2} ,
\end{align*}
for $C_{K-1}$ the volume of the unit $\ell_2$ ball in $\R^{K-1}$. 
For $i \in [K]$, let $p'_i = p_i (1 + \delta_i)$. Also, let $i_t$ satisfy $y_t = e_{i_t}$. 
We can 
show that $\max_{i \neq i_t} |\delta_i| \leq \frac{1}{t-1} \leq \frac{1}{c_{t-1,i_t}}$ and $|\delta_{i_t}| \leq \frac{1}{c_{t-1,i_t}}$. 
Thus, 
$p$ and $p'$ are entry-wise close in a relative sense, 
allowing us to show the relative volume difference is 
$O(\frac{K}{c_{t-1,i_t}})$.

Next, we explain how we control the relative volume of the symmetric difference of $E(p)$ and $E(p')$. Define $Q_p(x) = \sum_{i=1}^K (x_i/p_i - 1)^2$ and $Q_{\sup} = \sup_{x \in E(p) \cup E(p')} |Q_{p'}(x) - Q_p(x)|$. Observe that $E(p) = \{x \in \Delta_K \colon 0 \leq Q_p(x) \leq \sigma^2 \}$. It is possible to show that a scaled up version of $E(p)$ contains $E(p')$ while a scaled down version of $E(p)$ is contained in 
$E(p) \cap E(p')$. Therefore, the symmetric difference is contained in the ``shell'' formed by the boundaries between the scaled up version and scaled down version of $E(p)$. In particular, we use the shell
\begin{align*}
S(p, Q_{\sup}) = \{x \in \Delta_K \colon \sigma^2 - Q_{\sup} \leq Q_p(x) \leq \sigma^2 + Q_{\sup} \} .
\end{align*}
The relative volume of the symmetric difference therefore satisfies
\begin{align*}
\textstyle 
\frac{V(E(p) \symdiff E(p'))}{V(p)} 
\leq \frac{V(S(p, Q_{\sup}))}{V(p)} 
= \left( 1 + \frac{Q_{\sup}}{\sigma^2} \right)^{(K-1)/2} 
  - \left( 1 - \frac{Q_{\sup}}{\sigma^2} \right)^{(K-1)/2} .
\end{align*}
Finally, using the convexity of $b \mapsto b^{K-1}$, we show the above to be at most $e^{\frac{K Q_{\sup}}{2 \sigma^2}} \cdot \frac{2 K Q_{\sup}}{\sigma^2}$. By carefully controlling $Q_{\sup}$ using that all $i \neq i_t$ satisfy 
$|\delta_i| \leq \frac{1}{t-1}$ 
while $i_t$ satisfies 
$|\delta_{i_t}| \leq \frac{1}{c_{t-1,i_t}}$, 
we are able to show that $\frac{K Q_{\sup}}{2 \sigma^2} \leq 1$ (which controls the exponential term) and 
$Q_{\sup} = O\rbr{ \frac{\sqrt{K} \sigma}{t} + \frac{\sigma}{c_{t-1,i_t}} }$. 
Basic algebra then shows that the relative volume of the symmetric difference is 
$O\rbr{ \frac{K^{3/2}}{\sigma t} + \frac{K}{\sigma c_{t-1,i_t}} }$, 
which dominates 
our bound on 
the relative volume difference and matches the stated bound.
\end{proof} 

\section{Discussion and future work}
\label{sec:discussion}

\paragraph{Optimal dependence on $K$.} 
 In light of prior results from \cite{luo2024optimal}, for general, bounded losses we do not obtain the optimal dependence on $K$. We initially tried to use FTPL with Gaussian perturbations, but due to issues in extending the loss and dealing with projections, we ultimately adopted self-concordant noise. If one ignores the issues with perturbed probability vectors being outside the simplex (or put differently, if one allows improper learning), then $O(\sqrt{T K})$ regret for general losses is indeed achievable using Gaussian noise; see \pref{app:gaussian} for details. 

\paragraph{Simultaneous optimality.}
While our work establishes promising results, our initial motivating question is still wide open:
does there exist an online forecasting algorithm which achieves the optimal regret rate (up to constant or logarithmic factors in $T$) for \textit{every} proper loss?
Even setting aside such simultaneous optimality, even the question of characterizing which proper losses have which optimal regret rates is still open, as far as we know.

\acks{We thank Bobby Kleinberg and Bo Waggoner for helpful discussions. HL is supported by NSF award IIS-1943607. NM was supported by the NSERC Discovery Grant RGPIN-2025-05257.}

\bibliography{refs}

\appendix

\section{Additional Related Work} \label{app:more-related-work}

\paragraph{Prediction with Expert Advice} 
In the game of prediction with expert advice, in each round, each of a collection of experts provides the learning algorithm with advice (an action in some action space). For a finite number $N$ of experts, 
\citet[Corollary 2]{chernov2009prediction} showed that a single algorithm can simultaneously get $O(\log(N))$ regret against all experts, where for each expert, the regret can be measured using a potentially different, proper, \emph{mixable} \citep{vovk1998game} loss function (zero-one loss and, more generally, V-shaped losses are not mixable). In contrast to their setting, in our work one can view each element of the simplex as an expert (so, in general, we have infinitely many experts, and these experts are constant experts\footnote{A constant expert is an expert whose advice does not change from round to round.}). In a slightly later paper with binary outcomes --- and still with finitely many experts --- \citet[Theorem 12]{chernov2010prediction} show that it is possible to get $O(1)$ regret against each expert under squared loss and $\otil(\sqrt{T})$ expected regret against each expert under zero-one loss, where in both cases there is a logarithmic dependence on $N$; however, as we now explain, this result is not a simultaneous regret result due to an important nuance. The reason their algorithm does not fit the simultaneous regret paradigm we consider here is that, when predicting $p = 1/2$ for squared loss, their algorithm is allowed submit a different (randomized in a way that the algorithm wishes) prediction for 0-1 loss; we do not have this freedom in our paradigm. 
Both these papers of Chernov and Vovk use the technique of defensive forecasting. Although there are large differences in the setting --- for one, the bounds of Chernov and Vovk would become infinite in our setting --- it is interesting that prior work over 15 years ago had considered notions close to simultaneous regret.

\paragraph{Self-Concordant Barriers}
Self-concordant barriers play a central role in convex optimization such as interior-point methods~\citep{nesterov1994interior}.
\citet{abernethy2009competing} showed that they can be used as a regularizer in the standard Follow-the-Regularized-Leader (FTRL) framework to handle the exploration issue in adversarial linear bandits.
In particular, their algorithm explores the surface of a Dikin ellipsoid centered at the decision point output by FTRL. 
Our self-concordant noise of \pref{alg:multiclass} is closely related to this idea --- it can be shown that our prediction $P_t$ is sampled from a scaled-down version of a Dikin ellipsoid centered at $\pbar_{t-1}$ (that is, $E(\pbar_{t-1})$ defined in \pref{sec:stability}).

\section{Omitted proofs from \pref{sec:warm-up}}

We prove \pref{lem:squared-loss-decomp}.

\begin{proof}
  We have
\begin{align*}
 \E [ (P_t-y_t)^2 ] &= \E [ (P_t-\mu_t+\mu_t-y_t)^2 ] \\
 &=  (\mu_t-y_t)^2 + \E [ (P_t-\mu_t)^2] + \E [ (P_t-\mu_t)(\mu_t-y_t)] \\
 &= (\mu_t-y_t)^2 + \sigma^2_t,
\end{align*}
and thus
  \begin{align*}
    \Reg_{\ellsq}
    &= \sum_{t=1}^T \E[(P_t-y_t)^2] - \sum_{t=1}^T (\bar p_T - y_t)^2
    \\
    &= \sum_{t=1}^T (\mu_t-y_t)^2 + \sigma^2_t - \sum_{t=1}^T (\bar p_T - y_t)^2
    \\
    &= \Reg_{\ellsq}(\{\mu_t\}_t) + \sum_{t=1}^T \sigma^2_t~.
  \end{align*}
\end{proof}

\section[Remaining proof for $\beta$-smooth proper losses]{Remaining proof for $\boldsymbol{\beta}$-smooth proper losses}
\label{app:beta-smooth}

The following lemma was used in the proof of \pref{prop:multiclass_smooth}.

\begin{lemma} \label{lem:subgaussian-sphere}
It holds that
\begin{align*}
\E_{S_t}\sbr{\|S_t\|_\infty^2} 
\leq \frac{8 \sigma^2 \log K}{K-1} .
\end{align*}
\end{lemma}
\begin{proof}
Without loss of generality, we prove the claim for $\sigma = 1$. 
We write $S$ instead of $S_t$.

First, we define $\tilde{S} = \frac{S}{\|S\|}$, which can only be larger coordinate-wise than $S$. Note that $\tilde{S}$ is distributed according to the uniform distribution on a great circle of the sphere $\mathbb{S}^K_2$. Therefore, an equivalent way to generate $\tilde{S}$ is to first draw a random variable $X$ from the uniform distribution on $\mathbb{S}^{K-1}_2 \times \{0\} \subset \R^K$ and to then, for a suitable rotation matrix $R$ in the special orthogonal group $\mathrm{SO}(K)$, set $\tilde{S} = R X$. With this construction, it follows that
\begin{align*}
\E_{S}\sbr{\|S\|_\infty^2} 
\leq \E_{X}\sbr{\|R X\|_\infty^2} 
= \E_{X}\sbr{ \max_{j \in [K]} (R X)_j^2 } .
\end{align*}

From Theorem 3.4.5 of \cite{vershynin2025high}, $X$ is $\frac{1}{K-1}$-subgaussian, meaning that for all unit vectors $v \in \R^K$, it holds that
\begin{align*}
\Pr \left( \langle v, X \rangle > u \right) \leq 2 \exp \left(-\frac{u^2 (K-1)}{2} \right) .
\end{align*}
Since $R$ is orthogonal, considering $\langle v, R X \rangle = \langle R^T v, X \rangle$, it follows that $R X$ also $\frac{1}{K-1}$-subgaussian. In particular, for all $j \in [K]$, the random variable $(R X)_j = \langle e_j, R X \rangle$ is $\frac{1}{K-1}$-subgaussian. 

Let $W = R X$. 
Therefore,
\begin{align*}
\Pr \left( \max_{j \in [K]} W_j^2 \geq u \right) 
&\leq K \max_{j \in [K]} \Pr \left( W_j^2 \geq u \right) \\
&= K \max_{j \in [K]} \Pr \left( |W_j| \geq \sqrt{u} \right) \\
&\leq 2 K \exp \left( -\frac{u (K-1)}{2} \right) .
\end{align*}

Hence, for any $\alpha \geq 0$,
\begin{align*}
\E \left[ \max_{j \in [K]} W_j^2 \right] 
&= \int_0^\infty \Pr \left( \max_{j \in [K]} W_j^2 \geq u \right) du \\
&= \int_0^\alpha \Pr \left( \max_{j \in [K]} W_j^2 \geq u \right) du 
  + \int_\alpha^\infty \Pr \left( \max_{j \in [K]} W_j^2 \geq u \right) du \\
&\leq \alpha
      + 2 K \int_\alpha^\infty \exp \left( -\frac{u (K-1)}{2} \right) du \\
&\leq \alpha
      + \left[ -\frac{4 K}{K-1} \exp \left( -\frac{u (K-1)}{2} \right) \right|^\infty_\alpha \\
&= \alpha
   + \frac{4 K}{K-1} \exp \left( -\frac{\alpha (K-1)}{2} \right) .
\end{align*}
Setting $\alpha = \frac{2 \log K}{K-1}$ gives the bound
\begin{align*}
\frac{2 \log K}{K-1} + \frac{4}{K-1} 
\end{align*}
which, for $K \geq 2$, is at most $\frac{8 \log K}{K-1}$.
\end{proof}

\section{Analysis of stability term} \label{app:stability-new}

We begin by stating the following generic stability result that was used at the start of \pref{sec:stability}.

\begin{lemma} \label{lem:tv-stability}
For all $t \in \{1, \ldots, T+1\}$, let $F_t$ be the law of the random variable $P_t$. Then for any loss function taking values in $[-1, 1]$,
\begin{align*}
\textstyle 
\E\sbr{\sum_{t=1}^T \ell(P_t, y_t) - \ell(P_{t+1}, y_t)}
\leq 2\sum_{t=1}^T  \|F_{t+1} - F_t\|_{\mathrm{TV}} ,
\end{align*}
where $\|\cdot\|_{\mathrm{TV}}$ is the total variation distance.
\end{lemma}

\begin{proof}
Let $f_t$ and $f_{t+1}$ be the density functions of $F_t$ and $F_{t+1}$ respectively.
Then 
\begin{align*}
\E\sbr{\ell(P_t, y_t) - \ell(P_{t+1}, y_t)} & =\int_{\Delta_K} \ell(p,y_t)(f_t(p)-f_{t+1}(p))dp \\
&\leq \int_{\Delta_K} |f_t(p)-f_{t+1}(p)|dp = 2\|F_{t+1} - F_t\|_{\mathrm{TV}}.
\end{align*}
Summing over $t$ finishes the proof.
\end{proof}

\begin{proof}(of \pref{thm:ellipsoid-tv-new-main-text})
Throughout the proof, we adopt the notation $\rho_{\min} = \frac{1}{t-1}$ and $\rho_{\max} = \frac{1}{c_{t-1,i_t}}$.

First, for $i \in [K]$, let $p'_i = p_i (1 + \delta_i)$. Also, let $i_t$ satisfy $y_t = e_{i_t}$. 
\pref{lem:rho-min-max} implies for all $i \neq i_t$ that $|\delta_i| \leq \rho_{\min}$, and the same lemma implies that $|\delta_{i_t}| \leq \rho_{\max}$. 
Observe that $\rho_{\min} \leq \rho_{\max}$.

Let $V(A)$ be the $(K-1)$-dimensional volume of a set $A$, and $V(p)$ be a shorthand of $V(E(p))$. 
As shown in \pref{lem:tv-unif}, the TV distance can be bounded as:
\begin{align*}    
2 d(p, p') 
\leq \frac{|V(p') - V(p)|}{\max\{V(p), V(p')\}} 
     + \frac{V(E(p) \symdiff E(p'))}{\min\{V(p), V(p')\}} .
\end{align*}

\paragraph{Step 1: Bounding the Relative Volume Difference}

As we show in \pref{lem:newest-volume}, the volume of $E(p)$ is given by
\begin{align*}
V(p) 
= C_{K-1} \cdot \sigma^{K-1}  
  \left( \prod_{i=1}^K p_i \right) 
  \left( \frac{1}{K} \sum_{i=1}^K \frac{1}{p_i^2} \right)^{1/2} ,
\end{align*}
where $C_{K-1}$ is the volume of the unit $(K-1)$-ball. 
We analyze the logarithmic difference:
\begin{align*}
\log \frac{V(p')}{V(p)} 
= \sum_{i=1}^K \log(1+\delta_i) 
  + \frac{1}{2} 
    \log\left(
        \frac{\sum_{i=1}^K \frac{1}{(p_i')^2}}
             {\sum_{i=1}^K \frac{1}{p_i^2}} \right).
\end{align*}

Since $| \delta_i | \leq \rho_{\max} < 1/2$, we have $| \log(1 + \delta_i)| \leq 2 \rho_{\max}$. 
Thus, $\left| \sum_{i=1}^K \log(1 + \delta_i) \right| \leq 2 K \rho_{\max}$.
Also,
\begin{align*}
\frac{1}{(p'_i)^2} 
= \frac{1}{(p_i (1 + \delta_i))^2} 
\leq \frac{1}{(1 - \rho_{\max})^2} \frac{1}{p_i^2}
\end{align*}
and hence 
\begin{align*}
\frac{1}{2} 
\log\left(
    \frac{\sum_{i=1}^K \frac{1}{(p_i')^2}}
         {\sum_{i=1}^K \frac{1}{p_i^2}} \right)
\leq \log \frac{1}{1 - \rho_{\max}}
= \log \left(1 + \frac{\rho_{\max}}{1 - \rho_{\max}} \right)
\leq \frac{\rho}{1 - \rho_{\max}}
\leq 2 \rho_{\max} .
\end{align*}
Therefore, $\left| \log \frac{V(p')}{V(p)} \right| \leq 2 K \rho_{\max} + 2 \rho_{\max} \leq 3 K \rho_{\max} \leq 1$. 
The relative volume difference is thus bounded as
\begin{align*}
\frac{|V(p') - V(p)|}{V(p)} 
&= \left| e^{\log(V(p')/V(p))} - 1 \right| \\
&\leq \left| e^{3 K \rho_{\max}} - 1 \right| \\
&\leq 6 K \rho_{\max} .
\end{align*}

\paragraph{Step 2: Bounding the Volume of the Symmetric Difference} 

Assume $V(p) \leq V(p')$ without loss of generality. 
Let $Q_p(x) = \sum_{i=1}^K (x_i/p_i - 1)^2$
and $Q_{\sup} = \sup_{x \in E(p) \cup E(p')} \left| Q_{p'}(x) - Q_p(x) \right|$.
By \pref{lem:Qsup-new}, we have
\begin{align}
Q_{\sup} 
&\leq 16 K \rho_{\min}^2 
      + 36 \sqrt{K} \rho_{\min} \sigma 
      + 38 \rho_{\max} \sigma \nonumber \\
&\leq \frac{1.34 \sigma^2}{K} \label{eqn:Qsup-bound-simpler} \\
&\leq \sigma^2 \nonumber ,
\end{align}
where the second inequality follows because 
$\rho_{\min} \leq \frac{\sigma}{54 K^{3/2}}$ and $\rho_{\max} \leq \frac{\sigma}{57 K}$. 
Now, we claim that the symmetric difference $E(p) \symdiff E(p')$ is contained in the shell
\begin{align*}
S(p, Q_{\sup}) 
= \left\{
        x \in \Delta_K : \sigma^2 - Q_{\sup} \leq Q_p(x) \leq \sigma^2 + Q_{\sup} 
    \right\}.
\end{align*}
Indeed, if $x \in E(p) \setminus E(p')$, then 
$Q_p(x) \leq \sigma^2$ 
and 
$Q_p(x) \geq Q_{p'}(x) - Q_{\sup} > \sigma^2 - Q_{\sup}$.
Similarly, if $x \in E(p') \setminus E(p)$, then 
$Q_p(x) > \sigma^2$ 
and $Q_p(x) \leq Q_{p'} + Q_{\sup} \leq \sigma^2 + Q_{\sup}$.

This implies:
\[
\frac{V(E(p) \symdiff E(p'))}{\min\{V(p), V(p')\}} \leq \frac{V(S(p, Q_{\sup}))}{V(p)}.
\]
To proceed, we note that the relative volume of the shell is:
\begin{align}
\frac{V(S(p, Q_{\sup}))}{V(p)} 
&= \left(1 + \frac{Q_{\sup}}{\sigma^2}\right)^{(K-1)/2} - \left(1 - \frac{Q_{\sup}}{\sigma^2}\right)^{(K-1)/2}.\notag
\end{align}
Using the the fact that $a^{K-1} - b^{K-1} \leq (K-1)a^{K-2}(a-b)$ due to convexity,
we plug in $a = \sqrt{1 + \frac{Q_{\sup}}{\sigma^2}}$ and $b = \sqrt{1-\frac{Q_{\sup}}{\sigma^2}}$ to arrive at
\begin{align*}
\frac{V(S(p, Q_{\sup}))}{V(p)} 
&\leq (K-1)\left(1 + \frac{Q_{\sup}}{\sigma^2}\right)^{\frac{K}{2}-1}\left(\sqrt{1 + \frac{Q_{\sup}}{\sigma^2}} - \sqrt{1 - \frac{Q_{\sup}}{\sigma^2}}\right) \\
&\leq K e^{\frac{K Q_{\sup}}{2\sigma^2}}\cdot \frac{\frac{2Q_{\sup}}{\sigma^2}}{\sqrt{1 + \frac{Q_{\sup}}{\sigma^2}}+\sqrt{1 - \frac{Q_{\sup}}{\sigma^2}}} \\
&\leq e^{\frac{K Q_{\sup}}{2\sigma^2}}\cdot \frac{2 K Q_{\sup}}{\sigma^2}.
\end{align*}

Recall from \eqref{eqn:Qsup-bound-simpler} that $Q_{\sup} \leq \frac{1.34 \sigma^2}{K}$, and so $\frac{K Q_{\sup}}{2 \sigma^2} \leq 0.67$, making $e^{\frac{K Q_{\sup}}{2 \sigma^2}} = O(1)$.

We then have
\begin{align*}
\frac{V(E(p) \symdiff E(p'))}{\min\{V(p), V(p')\}} 
\leq \frac{V(S(p, Q_{\sup}))}{V(p)} 
&= O\rbr{
      \frac{K^2 \rho_{\min}^2}{\sigma^2} 
      + \frac{K^{3/2} \rho_{\min}}{\sigma} 
      + \frac{K \rho_{\max}}{\sigma}
  } \\
&= O\rbr{
      \frac{K^{3/2} \rho_{\min}}{\sigma} 
      + \frac{K \rho_{\max}}{\sigma} 
  } ,
\end{align*}
where the last line follows from our assumption $\rho_{\min} \leq \frac{\sigma}{54 K^{3/2}}$.

\paragraph{Step 3: Conclusion}

Combining the bounds, we have:
\begin{align*}
d(p, p') 
&= 
O\rbr{
    K \rho_{\max}
    + \frac{K^{3/2} \rho_{\min}}{\sigma} 
    + \frac{K \rho_{\max}}{\sigma}
} \\
&= 
O\rbr{ 
    \frac{K^{3/2} \rho_{\min}}{\sigma} 
    + \frac{K \rho_{\max}}{\sigma}
} ,
\end{align*}
which completes the proof.
\end{proof}

\section{Regret of BTPL (Proof of \pref{lem:BTPL_regret})}
\label{app:btpl}

We first present the following useful lemma.

\begin{lemma}\label{lem:BTPL-direct}
For any proper loss $\ell \in \calL$ and for any sequence $Z_1, \ldots, Z_{T+1} \in \R^K$ satisfying $P_t\triangleq\bar{p}_{t-1} + Z_t \in \Delta_K$, it holds that
\begin{align*}
\sum_{t=1}^T \bigl( \loss(P_{t+1}, y_t) - \loss(\pbar_{T}, y_t) \bigr)
\leq 2 \sum_{t=1}^T \|t Z_{t+1} - (t-1) Z_{t}\|_1 .
\end{align*}
\end{lemma}

\begin{proof} 
Let $\lossvec_p \in \R^K$ be the vector satisfying $\langle \lossvec_p, e_j \rangle = \loss(p, e_j)$ for $j \in [K]$. 
Observe that
\begin{align*}
P_{t+1} \in 
\argmin_{p \in \Delta_K} \left\langle \lossvec_p, \bar{p}_t + Z_{t+1} \right\rangle 
&= \argmin_{p \in \Delta_K} 
          \left\langle \lossvec_p, \sum_{s=1}^t y_s + t  Z_{t+1} \right\rangle \\
&= \argmin_{p \in \Delta_K}
          \sum_{s=1}^t \bigl\langle
                                    \lossvec_p,
                                     y_s + s Z_{s+1} - (s-1) Z_s
                                \bigr\rangle \\
&= \argmin_{p \in \Delta_K}
          \sum_{s=1}^t \bigl\langle \lossvec_p, y_s + \nu_s - \nu_{s-1} \bigr\rangle ,
\end{align*}
where we introduce the notation $\nu_s = s Z_{s+1}$.

Therefore, by the standard Be-the-Leader lemma (see e.g.,~\citealt[Lemma~3.1]{cesabianchi2006prediction}), we have
\begin{align*}
\sum_{t=1}^T \bigl\langle \lossvec_{P_{t+1}}, y_t + \nu_t - \nu_{t-1} \bigr\rangle
- \sum_{t=1}^T \bigl\langle \lossvec_{\bar{p}_T}, y_t + \nu_t - \nu_{t-1} \bigr\rangle 
\leq 0  ,
\end{align*}
or equivalently,
\begin{align*}
\sum_{t=1}^T \ell(P_{t+1}, y_t) - \sum_{t=1}^T \ell(\bar{p}_T, y_t) 
&\leq \sum_{t=1}^T \langle  \lossvec_{\bar{p}_T}-\lossvec_{P_{t+1}}, \nu_t - \nu_{t-1} \rangle \\
&\leq \sum_{t=1}^T \| \lossvec_{\bar{p}_T}-\lossvec_{P_{t+1}} \|_\infty \, \| \nu_t - \nu_{t-1} \|_1 \\
&\leq 2 \sum_{t=1}^T \| \nu_t - \nu_{t-1} \|_1 ,
\end{align*}
where the last step is because the range of $\ell$ is $[-1, 1]$.
\end{proof}

Before we use this lemma to prove \pref{lem:BTPL_regret}, we need to introduce the following preliminaries.
\paragraph{Local rotations.}
Consider the unit sphere in $\R^K$, viewed as a $(K-1)$-dimensional Riemannian manifold equipped with the standard Riemannian metric; any geodesic is an arc of a great circle. For any unit vectors $p, q \in \R_+^K$, there is a unique geodesic from $p$ to $q$. Let $\rot{p}{q}$ be the element of the special orthogonal group $\mathrm{SO}(K)$ 
that is the unique rotation matrix that sends $p$ to $q$ by traveling along this geodesic. In particular, $\rot{p}{q}$ 
satisfies
\begin{enumerate}
\item $\rot{p}{q} p = q$;
\item $\rot{p}{q}$ rotates in the span of $p$ and $q$.
\end{enumerate}
We then extend the definition of $R$ to any $p, q \in \R_+^K$ via $R(p, q) = R\rbr{\frac{p}{\|p\|_2}, \frac{q}{\|q\|_2}}$.

\paragraph{Incrementally updated self-concordant noise.}
The original sampling process for self-concordant noise involved the following two steps. First, sample $S_t$ by drawing it uniformly at random from the intersection of the $K$-dimensional $\ell_2$ ball with radius $\sigma$ and the subspace $\left\{ s \in \R^K \colon \langle s, \bar{p}_{t-1}\rangle = 0 \right\}$. Next, set
$
Z_t = \mathrm{diag}(\bar{p}_{t-1}) S_t.
$
Now, consider changing the generation of the $S_t$ sequence to the following (while keeping the same formula for $Z_t$): 
first, samples $S_1$ in the same way; then for each $t$, let $S_{t+1} = \rot{\bar{p}_{t-1}}{\bar{p}_{t}} S_t$.
It is clear that the marginal distribution of $Z_{t}$ remains the same,
which means it is enough to analyze this different procedure (since we assume an oblivious adversary). \\

\begin{proof}(of \pref{lem:BTPL_regret})
From \pref{lem:BTPL-direct},  it suffices to bound
\begin{align*}
\sum_{t=1}^T \| t Z_{t+1} - (t - 1) Z_{t} \|_1 .
\end{align*}
Observe that
\begin{align}
\| t Z_{t+1} - (t - 1) Z_{t} \|_1  
&= \left\| t Z_{t+1} - (t - 1) Z_{t+1} + (t - 1) Z_{t+1} - (t - 1) Z_{t} \right\|_1 \nonumber \\
&\leq \| Z_{t+1} \|_1 + (t - 1) \| Z_{t+1} - Z_{t} \|_1 . \label{eqn:nu_t-diff}
\end{align}

The first term is bounded as
\begin{align}
\| Z_{t+1} \|_1 
= \| \mathrm{diag}(\bar{p}_t) S_{t+1} \|_1 
= \sum_{j=1}^K (\bar{p}_t)_j \cdot |S_{t+1,j}| 
\leq \|S_{t+1} \|_\infty 
\leq \sigma . \label{eqn:easy-term}
\end{align}

Next, let us bound $\| Z_{t+1} - Z_{t} \|_1$ in \eqref{eqn:nu_t-diff}. 
Letting $D_t = \mathrm{diag}(\bar{p}_t)$, we have 
\begin{align}
\left\| Z_{t+1} - Z_{t} \right\|_1 
&= \left\| D_t S_{t+1} - D_{t-1} S_{t} \right\|_1 \nonumber \\
&= \left\| D_t S_{t+1} - D_{t-1} S_{t+1} + D_{t-1} S_{t+1} - D_{t-1} S_{t} \right\|_1 \nonumber \\
&\leq \left\| (D_t - D_{t-1}) S_{t+1} \right\|_1 + \left\| D_{t-1} (S_{t+1} - S_{t}) \right\|_1 . \label{eqn:first-and-second-term}
\end{align}
We control the first and second terms in turn. 
For the first term,
$\| \bar{p}_t - \bar{p}_{t-1} \|_\infty \leq \frac{1}{t}$ implies that
\begin{align}
\left\| (D_t - D_{t-1}) S_{t+1} \right\| 
\leq \frac{1}{t} \| S_{t+1} \|_1 
\leq \frac{\sqrt{K}\sigma}{t} . \label{eqn:first-term-bound}
\end{align}

The second term in \eqref{eqn:first-and-second-term} requires more work.
Observe that 
\begin{align}
\left\| D_{t-1} (S_{t+1} - S_{t}) \right\|_1 
&= {\textstyle \E_{j \sim \bar{p}_{t-1}} } \left[ \left| (S_{t+1} - S_{t})_j \right| \right] \nonumber \\
&\leq \left\| S_{t+1} - S_{t} \right\|_\infty \nonumber \\
&\leq \left\| S_{t+1} - S_{t} \right\|_2 \label{eqn:early-step}
\end{align}
Now, for any $t \geq 1$, recall that we take $S_{t+1} = \rot{\bar{p}_{t-1}}{\bar{p}_t} S_{t}$. 
To avoid notation clutter, let $\|\cdot\|$ denote the $\ell_2$-norm $\|\cdot\|_2$. 
Then
\begin{align}
\begin{aligned}
\left\| S_{t+1} - S_{t}\right\| 
&= \left\| \rot{\bar{p}_{t-1}}{\bar{p}_t} S_{t} - S_{t} \right\| \\
&\leq \sigma \left\| \rot{\bar{p}_{t-1}}{\bar{p}_t} \frac{S_{t}}{\|S_{t}\|} - \frac{S_{t}}{\|S_{t}\|} \right\| .
\end{aligned} \label{eqn:II-starting point}
\end{align}

Now, observe that since $\rot{\bar{p}_{t-1}}{\bar{p}_t}$ is a rotation matrix, we can swap $\frac{S_{t}}{\|S_{t}\|}$ with any unit vector without changing the value of the expression. In particular, we replace it with $\frac{\bar{p}_{t-1}}{\|\bar{p}_{t-1}\|}$, giving

\begin{align}
& \left\| \rot{\bar{p}_{t-1}}{\bar{p}_t} \frac{\bar{p}_{t-1}}{\|\bar{p}_{t-1}\|} - \frac{\bar{p}_{t-1}}{\|\bar{p}_{t-1}\|} \right\| \nonumber \\
&= \left\| \frac{\bar{p}_t}{\|\bar{p}_t\|} - \frac{\bar{p}_{t-1}}{\|\bar{p}_{t-1}\|} \right\| \nonumber \\
&= \left\| \frac{\bar{p}_t}{\|\bar{p}_t\|} 
           - \frac{\bar{p}_t}{\|\bar{p}_{t-1}\|}
           + \frac{\bar{p}_t}{\|\bar{p}_{t-1}\|}
           - \frac{\bar{p}_{t-1}}{\|\bar{p}_{t-1}\|} 
   \right\| \nonumber \\
&\leq \left\| \frac{\bar{p}_t}{\|\bar{p}_t\|} 
              - \frac{\bar{p}_t}{\|\bar{p}_{t-1}\|}
      \right\|
      + \left\| \frac{\bar{p}_t}{\|\bar{p}_{t-1}\|}
                - \frac{\bar{p}_{t-1}}{\|\bar{p}_{t-1}\|} 
        \right\| . \label{eqn:II-first-term-second-term}
\end{align}
We now bound the two terms in \eqref{eqn:II-first-term-second-term}. 

The first term in \eqref{eqn:II-first-term-second-term} is
\begin{align}
\left\| \frac{\bar{p}_t}{\|\bar{p}_t\|} 
              - \frac{\bar{p}_t}{\|\bar{p}_{t-1}\|}
\right\|
= \left| 1 - \frac{\|\bar{p}_t\|}{\|\bar{p}_{t-1}\|} \right| . \label{eqn:II-first-term-cleaner}
\end{align}
As shown by \pref{lem:II-first-term-new} (stated and proved after this proof),
\begin{align*}
\left| 1 - \frac{\|\bar{p}_t\|}{\|\bar{p}_{t-1}\|} \right| 
\leq \frac{2 \sqrt{K}}{t} .
\end{align*}

The second term in \eqref{eqn:II-first-term-second-term} is
\begin{align}
\frac{1}{\|\bar{p}_{t-1}\|} \cdot \|\bar{p}_t - \bar{p}_{t-1}\| 
\leq \sqrt{K} \cdot \frac{1}{t} . \label{eqn:II-second-term-bound}
\end{align}

Put together, this shows that \eqref{eqn:early-step} is of order $O(\sqrt{K}\sigma/t)$.
Combining everything, we have shown
\[
\sum_{t=1}^T \bigl( \loss(P_{t+1}, y_t) - \loss(\pbar_{T}, y_t) \bigr)
= O\rbr{T\sigma + T\sqrt{K}\sigma}= O\rbr{T\sqrt{K}\sigma},
\]
completing the proof.
\end{proof}

\begin{lemma} \label{lem:II-first-term-new}
It holds that
\begin{align*}
\left| 1 - \frac{\|\bar{p}_t\|}{\|\bar{p}_{t-1}\|} \right| 
\leq \frac{2 \sqrt{K}}{t} .
\end{align*}
\end{lemma}
\begin{proof}
Observe that
\begin{align*}
\bar{p}_t 
= \frac{(t - 1) \bar{p}_{t-1} + y_t}{t} 
= \bar{p}_{t-1} + \frac{y_t - \bar{p}_{t-1}}{t} .
\end{align*}
Therefore,
\begin{align*}
\left| 1 - \frac{\|\bar{p}_t\|}{\|\bar{p}_{t-1}\|} \right| 
= \left| \frac{\|\bar{p}_{t-1}\| - \|\bar{p}_t\|}{\|\bar{p}_{t-1}\|} \right| 
\leq \frac{\left\| \frac{y_t - \bar{p}_{t-1}}{t} \right\|}{\|\bar{p}_{t-1}\|}
\leq \frac{2}{t} \frac{1}{\|\bar{p}_{t-1}\|} 
\leq \frac{2\sqrt{K}}{t} .
\end{align*}

\end{proof}

\section{Technical lemmas} 
\label{app:technical-lemmas}

\begin{lemma} \label{lem:tv-unif}
Let $A$ and $B$ be sets with respective volumes $V(A)$ and $V(B)$, and $A \symdiff B$ be their symmetric difference with volume $V(A \symdiff B)$.
Let $P_A$ be the uniform distribution over $A$ and $P_B$ be the uniform distribution over $B$. 
Then 
\begin{align*}    
2 \|P_A - P_B\|_{\mathrm{TV}} 
\leq \frac{|V(A) - V(B)|}{\min\{V(A), V(B)\}} + \frac{V(A \symdiff B)}{\min\{V(A), V(B)\}}.
\end{align*}
\end{lemma}
\begin{proof}
Without loss of generality, assume that $V(A) \geq V(B)$. Then
\begin{align*}
2 \|P_A - P_B\|_{\mathrm{TV}} 
= \int_{A \cap B} 
      \left( \frac{1}{V(B)} - \frac{1}{V(A)} \right) dx 
  + \frac{V(A \setminus B)}{V(A)}
  + \frac{V(B \setminus A)}{V(B)}. 
\end{align*}
Clearly,
\begin{align*}
\frac{V(A \setminus B)}{V(A)}
+ \frac{V(B \setminus A)}{V(B)} 
\leq \frac{V(A \symdiff B)}{\min\{V(A), V(B)\}} ,
\end{align*}
and
\begin{align*}
\int_{A \cap B} 
    \left( \frac{1}{V(B)} - \frac{1}{V(A)} \right) dx 
&\leq \int_B \left( \frac{1}{V(B)} - \frac{1}{V(A)} \right) dx \\
&= 1 - \frac{V(B)}{V(A)} \\
&= \frac{V(A) - V(B)}{V(A)} \\
&= \frac{|V(A) - V(B)|}{\max\{V(A), V(B)\}} .
\end{align*}
Combining the two bounds above finishes the proof.
\end{proof}

\begin{lemma}\label{lem:newest-volume}
The volume of $E(p)$ is given by
\begin{align*}
V(p) 
= C_{K-1} \cdot \sigma^{K-1}  
  \left( \prod_{i=1}^K p_i \right) 
  \left( \frac{1}{K} \sum_{i=1}^K \frac{1}{p_i^2} \right)^{1/2} ,
\end{align*}
where $C_{K-1}$ is the volume of the unit $(K-1)$-dimensional $\ell_2$ ball.
\end{lemma}
\begin{proof}
We compute the volume of $E(p)$ in the case of $\sigma = 1$. The volume for general $\sigma$ is obtained by multiplying by $\sigma^{K-1}$.

In order to apply the standard formula for an ellipsoid, we need to put $E(p)$ into a $(K-1)$-dimensional parameterization.
To do so, we start with the unit $\ell_2$ ball in $\R^K$, denoted as $\fB_2^K$ in $\R^K$. Next, we linearly transform the ball using a matrix $D = \diag(p_1, \ldots, p_K)$, to create the ellipsoid $D \fB_2^K$ in $\R^K$. In preparation for the final step, define the hyperplane $H = \{ x \in \R^K \colon \langle \ones, x \rangle = 0 \}$. Let $U \in \R^{K \times (K-1)}$ be an orthonormal basis of $H$. Therefore, $U^T U = I$ and $U^T \ones = 0$. We now do the final step by projecting the ellipsoid $D \fB_2^K$ onto the hyperplane $H$,
leading to $UU^T D \fB_2^K \subseteq \R^{K}$,
which is clearly a shifted version of $E(p)$ (so that its center is at the origin).
As mentioned, what we need is the $(K-1)$-dimensional parameterization of this ellipsoid, which is $U^T D \fB_2^K \subseteq \R^{K-1}$.
Equivalently, this is the $(k-1)$-dimensional ellipsoid $\{y: y^\top(U^\top Q U)^{-1}y \leq 1\}$ where $Q = D^2$.
We now proceed to compute the volume of this ellipsoid using the standard formula $C_{K-1} \det(U^T Q U)^{1/2}$. The rest of the proof shows how to compute the determinant.

Let $w = \frac{1}{\sqrt{K}} \ones$, and define an orthonormal matrix $V = (U; w) \in \R^{K\times K}$. It will be useful to work with the matrix $M = V^T Q V$ and its inverse $M^{-1} = V^T Q^{-1} V$, the latter formula holding because $V$ is orthonormal. Before continuing, we express $M$ and $M^{-1}$ in block form as
\begin{align*}
M 
= V^T Q V 
= 
\begin{pmatrix}
U^T Q U & U^T Q w \\
w^T Q U & w^T Q w
\end{pmatrix}
\end{align*}
and 
\begin{align*}
M^{-1}
= V^T Q^{-1} V
= 
\begin{pmatrix}
U^T Q^{-1} U & U^T Q^{-1} w \\
w^T Q^{-1} U & w^T Q^{-1} w
\end{pmatrix}
\end{align*}

Next, we use a formula for expressing the inverse of $M$ in terms of its cofactor matrix $\mathrm{cof}(M)$ and determinant:
\begin{align*}
M^{-1} = \frac{\mathrm{cof}(M)^T}{\det(M)} .
\end{align*}
Now, on the one hand, $(M^{-1})_{K,K} = w^T Q^{-1} w$, while on the other hand, $(\mathrm{cof}(M)^T)_{K,K}$ is equal to the determinant of the minor of $M$ that has the last row and column of $M$ removed, i.e., $\det(U^T Q U)$. Therefore,
\begin{align*}
w^T Q^{-1} w 
= \frac{\det(U^T Q U)}{\det(V^T Q V)} 
= \frac{\det(U^T Q U)}{\det(Q)} ,
\end{align*}
where the second equality follows because $V$ is an orthonormal basis. 
Rearranging, we have $\det(U^T Q U) = \det(Q) w^T Q^{-1} w$.

Finally,
\begin{align*}
\det(U^T Q U)^{1/2} 
= \det(Q)^{1/2} (w^T Q^{-1} w)^{1/2} 
= \left( \prod_{i=1}^K p_i \right) 
  \left( \frac{1}{K} \sum_{i=1}^K \frac{1}{p_i^2} \right)^{1/2} .
\end{align*}
\end{proof}

\begin{lemma} \label{lem:rho-min-max}
Assume that $\{p, p'\} = \{\bar{p}_{t-1}, \bar{p}_t\}$, with both $p$ and $p'$ being in the relative interior of $\Delta_K$. 
Then for all $i \neq i_t$,
\begin{align*}
\max \left\{ 
         \left| \frac{p_i}{p'_i} - 1 \right| ,  
         \left| \frac{p'_i}{p_i} - 1 \right|
     \right\}
\leq \frac{1}{t-1}
\end{align*}
and
\begin{align*}
\max \left\{ 
         \left| \frac{p_{i_t}}{p'_{i_t}} - 1 \right| ,  
         \left| \frac{p'_{i_t}}{p_{i_t}} - 1 \right|
     \right\}
\leq \frac{1}{c_{t-1,i_t}} .
\end{align*}
\end{lemma}

\begin{proof}
Let $i \neq i_t$. Then
\begin{align*}
\bar{p}_{t,i} 
= \frac{(t - 1) \bar{p}_{t-1,i}}{t} .
\end{align*}
Hence,
\begin{align*}
\left| \frac{\bar{p}_{t,i}}{\bar{p}_{t-1,i}} - 1 \right|
= \left| \frac{t - 1}{t} - 1 \right| 
= \frac{1}{t} ,
\end{align*}
and
\begin{align*}
\left| \frac{\bar{p}_{t-1,i}}{\bar{p}_{t,i}} - 1 \right|
= \left| \frac{t}{t-1} - 1 \right| 
= \frac{1}{t-1} .
\end{align*}

Next, let $i = i_t$. 
Then
\begin{align*}
\bar{p}_{t,i} 
= \frac{(t - 1) \bar{p}_{t-1,i} + 1}{t} .
\end{align*}
It follows that
\begin{align*}
\left| \frac{\bar{p}_{t,i}}{\bar{p}_{t-1,i}} - 1 \right|
= \left| \frac{t - 1 + \frac{1}{\bar{p}_{t-1,i}}}{t} - 1 \right| 
= \left| \frac{\frac{1}{\bar{p}_{t-1,i}} - 1}{t} \right| 
= \left| \frac{\frac{t-1}{c_{t-1,i}} - 1}{t} \right| 
\leq \frac{t-1}{t c_{t-1,i}}
\leq \frac{1}{c_{t-1,i}} .
\end{align*}
Also, observe that
\begin{align*}
\bar{p}_{t-1,i} = \frac{t \bar{p}_{t,i} - 1}{t-1} ,
\end{align*}
so
\begin{align*}
\left| \frac{\bar{p}_{t-1,i}}{\bar{p}_{t,i}} - 1 \right|
= \left| \frac{t - \frac{1}{\bar{p}_{t,i}}}{t-1} - 1 \right| 
= \left| \frac{1 - \frac{1}{\bar{p}_{t,i}}}{t-1} \right| 
= \frac{\frac{t}{c_{t,i}} - 1}{t-1} 
= \frac{t - c_{t,i}}{c_{t,i} (t-1)} 
\leq \frac{1}{c_{t,i}} 
\leq \frac{1}{c_{t-1,i}} 
\end{align*}
\end{proof}

\begin{lemma} \label{lem:Qsup-new}
Assume that $\{p, p'\} = \{\bar{p}_{t-1}, \bar{p}_t\}$, with both $p$ and $p'$ being in the relative interior of $\Delta_K$. Define $\rho_{\min} = \frac{1}{t-1}$ and $\rho_{\max} = \frac{1}{c_{t-1,i_t}}$. Assume $\rho_{\min} \leq 1/2$ and  $\rho_{\max} \leq \sigma \leq 1/2$. 
Define $Q_p(x) = \sum_{i=1}^K (x_i/p_i - 1)^2$.
We have
\begin{align*}
Q_{\sup} \triangleq \sup_{x \in E(p) \cup E(p')} \left| Q_{p'}(x) - Q_p(x) \right| 
\leq 16 K \rho_{\min}^2 + 36 \sqrt{K} \rho_{\min} \sigma
+ 38 \rho_{\max} \sigma .
\end{align*}
\end{lemma}
\begin{proof}
Let $s_i = \frac{x_i}{p_i} - 1$ for $i \in [K]$, so that $Q_p(x) = \sum_{i=1}^K s_i^2$. 
We analyze the difference $Q_{p'}(x) - Q_p(x)$ for any $x \in E(p) \cup E(p')$:
\begin{align*} 
Q_{p'}(x) - Q_p(x) 
&= \sum_{i=1}^K \left( \frac{x_i}{p'_i} - 1 \right)^2 - s_i^2 \\
&= \sum_{i=1}^K \left( \frac{1 + s_i}{1 + \delta_i} - 1 \right)^2 - s_i^2 \\
&= \sum_{i=1}^K \frac{(s_i - \delta_i)^2 - s_i^2(1 + \delta_i)^2}{(1 + \delta_i)^2} \\
&= \sum_{i=1}^K \frac{-2 s_i \delta_i(1 + s_i) + \delta_i^2(1 - s_i^2)}{(1 + \delta_i)^2} .
\end{align*}

Next, let $i_t$ satisfy $y_t = e_{i_t}$. \pref{lem:rho-min-max} implies for all $i \neq i_t$ that $|\delta_i| \leq \frac{1}{t-1} = \rho_{\min}$, and the same lemma implies that $|\delta_{i_t}| \leq \frac{1}{c_{t-1,i_t}} = \rho_{\max}$.

Since for all $i \in [K]$, we have $(1 + \delta_i)^2 \geq (1 - \rho_{\max})^2 > 1/2$, 
it holds that
\begin{align}
&\frac{1}{2} \cdot | Q_{p'}(x) - Q_p(x) | \nonumber \\
&\leq \sum_{i=1}^K \left( 2 |s_i| \cdot |\delta_i| \cdot (1 + |s_i|) + \delta_i^2 \right) \nonumber \\ 
&= \sum_{i=1}^K \left( 2 |\delta_i| \cdot (|s_i| + s_i^2) + \delta_i^2 \right) \nonumber \\
&= \sum_{i \neq i_t} \left( 2 |\delta_i| \cdot (|s_i| + s_i^2) + \delta_i^2 \right) 
+ 2 |\delta_{i_t}| \cdot (|s_{i_t}| + s_{i_t}^2) + \delta_{i_t}^2 \nonumber \\
&\leq
\sum_{i \neq i_t} \left( 2 \rho_{\min} \cdot (|s_i| + s_i^2) + \rho_{\min}^2 \right) 
+ 2 \rho_{\max} \cdot (|s_{i_t}| + s_{i_t}^2) + \rho_{\max}^2 \nonumber \\
&\leq 2 \rho_{\min} \left( \|s\|_1 + \|s\|_2^2 \right) + K \rho_{\min}^2
      + 2 \rho_{\max} \left( \|s\|_2 + \|s\|_2^2 \right) + \rho_{\max}^2 \nonumber \\
&\leq 2 \sqrt{K} \rho_{\min} \cdot \|s\|_2 + 2 \rho_{\min} \cdot \|s\|_2^2 + K \rho_{\min}^2
      + 2 \rho_{\max} \left( \|s\|_2 + \|s\|_2^2 \right) + \rho_{\max}^2 . \label{eqn:almost-done-Qsup-new}
\end{align}

We now turn to bound $\|s\|_2$. First, if $x \in E(p)$, then clearly $\|s\|_2 \leq \sigma$ by definition of $E(p)$. On the other hand, if $x \in E(p')$, then from \pref{lem:s-norm-E-p'-new}, we have
\begin{align*}
\|s\|_2 \leq \sqrt{11 \sigma^2 + 2 K \rho_{\min}^2} .
\end{align*}
To conclude, $\|s\|_2 \leq \sqrt{11 \sigma^2 + 2 K \rho_{\min}^2}$ always holds. 
Using this fact, we bound the terms in \eqref{eqn:almost-done-Qsup-new}, starting with the terms involving $\rho_{\min}$. We have
\begin{align*}
&2 \sqrt{K} \rho_{\min} \cdot \|s\|_2 + 2 \rho_{\min} \cdot \|s\|_2^2 + K \rho_{\min}^2 \\
&\leq 2 \sqrt{K} \rho_{\min} \left( \sqrt{11} \sigma + \sqrt{2 K} \rho_{\min} \right) + 2 \rho_{\min} \left( 11 \sigma^2 + 2 K \rho_{\min}^2 \right) + K \rho_{\min}^2 \\
&\leq 4 K \rho_{\min}^2 + 2 \sqrt{11} \sqrt{K} \rho_{\min} \sigma + 22 \rho_{\min} \sigma^2 + 4 K \rho_{\min}^3 \\
&\leq 6 K \rho_{\min}^2 + 2 \sqrt{11} \sqrt{K} \rho_{\min} \sigma + 22 \rho_{\min} \sigma^2  && (\rho_{\min} \leq 1/2) \\
&\leq 6 K \rho_{\min}^2 + 2 \sqrt{11} \sqrt{K} \rho_{\min} \sigma + 11 \rho_{\min} \sigma  && (\sigma \leq 1/2) \\
&\leq 6 K \rho_{\min}^2 + 15 \sqrt{K} \rho_{\min} \sigma  && (K \geq 2) .
\end{align*}

Next, we bound the terms in \eqref{eqn:almost-done-Qsup-new} involving $\rho_{\max}$ as
\begin{align*}
&2 \rho_{\max} \left( \|s\|_2 + \|s\|_2^2 \right) + \rho_{\max}^2 \\
&\leq 2 \sqrt{11} \rho_{\max} \sigma + 2 \sqrt{2} \sqrt{K} \rho_{\max}  \rho_{\min} + 22 \rho_{\max} \sigma^2 + 4 K \rho_{\max}  \rho_{\min}^2 + \rho_{\max}^2 \\
&\leq \rho_{\max} \sigma \left( 2 \sqrt{11} + 11 + 1 \right) + 2 \sqrt{2} \sqrt{K} \rho_{\max} \rho_{\min} + 4 K \rho_{\max} \rho_{\min}^2  && (\rho_{\max} \leq \sigma \leq 1/2) \\
&\leq 19 \rho_{\max} \sigma + 3 \sqrt{K} \rho_{\min} \sigma + 2 K \rho_{\min}^2  && (\rho_{\max} \leq 1/2) .
\end{align*}
Adding the two results gives
\begin{align*}
&6 K \rho_{\min}^2 + 15 \sqrt{K} \rho_{\min} \sigma
+ 19 \rho_{\max} \sigma + 3 \sqrt{K} \rho_{\min} \sigma + 2 K \rho_{\min}^2 \\
&\leq 8 K \rho_{\min}^2 + 18 \sqrt{K} \rho_{\min} \sigma
+ 19 \rho_{\max} \sigma .
\end{align*}
Hence, it holds that
\begin{align*}
| Q_{p'}(x) - Q_p(x) | 
\leq 16 K \rho_{\min}^2 
     + 36 \sqrt{K} \rho_{\min} \sigma
     + 38 \rho_{\max} \sigma ,
\end{align*}
as desired.
\end{proof}

\begin{lemma} \label{lem:s-norm-E-p'-new}
Assume that $\{p, p'\} = \{\bar{p}_{t-1}, \bar{p}_t\}$. Define $\rho_{\min} = \frac{1}{t-1}$ and $\rho_{\max} = \frac{1}{c_{t-1,i_t}}$. 
Let $x \in E(p')$. Let $s_i = \frac{x_i}{p_i} - 1$ for $i \in [K]$. 
Assume that $\rho_{\max} \leq \sigma \leq 1/2$. 
Then
\begin{align*}
\|s\|_2^2 \leq 11 \sigma^2 + 2 K \rho_{\min}^2 .
\end{align*}
\end{lemma}
\begin{proof}
For all $i \in [K]$, define $\delta_i$ such that $p'_i = p_i (1+\delta_i)$. Observe that
\begin{align*}
\left| \frac{x_i}{p_i} - 1 \right| 
&= \left| \frac{x_i}{p'_i} \cdot \frac{p'_i}{p_i} - 1 \right| \\
&= \left| \frac{x_i}{p'_i} \left( \frac{p'_i}{p_i} - 1 \right) + \frac{x_i}{p'_i} - 1 \right| \\
&= \left| \frac{x_i}{p'_i} \cdot \delta_i + \frac{x_i}{p'_i} - 1 \right| \\
&\leq |\delta_i| \cdot \left| \frac{x_i}{p'_i} - 1 + 1 \right| 
      + \left| \frac{x_i}{p'_i} - 1 \right| \\
&\leq \left( |\delta_i| + 1 \right) \cdot \left| \frac{x_i}{p'_i} - 1 \right| 
      + |\delta_i| .
\end{align*}

Next, let $i_t$ satisfy $y_t = e_{i_t}$. \pref{lem:rho-min-max} implies for all $i \neq i_t$ that $|\delta_i| \leq \frac{1}{t-1} = \rho_{\min}$, and the same lemma implies that $|\delta_{i_t}| \leq \frac{1}{c_{t-1,i_t}} = \rho_{\max}$. 
Consequently, for $x \in E(p')$, 
\begin{align*}
\|s\|_2^2 
&\leq 2 \sum_{i \neq i_t} 
          (\rho_{\min} + 1)^2 \left( \frac{x_i}{p'_i} - 1 \right)^2 
      + 2 (K - 1) \rho_{\min}^2 \\
&\quad 
      + 2 (\rho_{\max} + 1)^2 \left( \frac{x_{i_t}}{p'_{i_t}} - 1 \right)^2 
      + 2 \rho_{\max}^2 \\
&\leq 2 (\rho_{\min} + 1)^2 \sigma^2 + 2 (K - 1) \rho_{\min}^2
      + 2 (\rho_{\max} + 1)^2 \sigma^2 + 2 \rho_{\max}^2 \\
&\leq 4 \rho_{\min}^2 \sigma^2 + 4 \sigma^2 + 2 (K - 1) \rho_{\min}^2
      + 4 \rho_{\max}^2 \sigma^2 + 4 \sigma^2 + 2 \rho_{\max}^2 \\
&\leq 10 \sigma^2 + 4 \rho_{\min}^2 \sigma^2 + 2 (K - 1) \rho_{\min}^2 + 4 \rho_{\max}^2 \sigma^2 \\
&\leq 11 \sigma^2 + 4 \rho_{\min}^2 \sigma^2 + 2 (K - 1) \rho_{\min}^2 && (\rho_{\max} \leq \sigma \leq 1/2) \\
&\leq 11 \sigma^2 + 2 K \rho_{\min}^2 && (\sigma \leq 1/2) .
\end{align*}
\end{proof}

The following lemma will be used for our results in Appendix~\ref{app:gaussian} (on the performance of Follow-The-Perturbed-Leader with Gaussian noise). 

\begin{lemma}\label{lem:gaussian_tv}
Let $F = \mathcal{N}(\mu, \sigma^2 I_d)$ and $F' = \mathcal{N}(\mu', \sigma^2 I_d)$ be two $d$-dimensional symmetric Gaussians with equal variance $\sigma^2$ and means such that $\|\mu - \mu'\|_2 = \delta \le \Delta$. Then the Total Variation (TV) distance between $F$ and $F'$ is bounded by $TV(F, F') = O(\Delta / \sigma)$.
\end{lemma}
\begin{proof}
Note that the total variation distance between two probability measures is invariant under isometries (translations and rotations) of the underlying space. We therefore first translate the coordinate system by $-\mu$. This maps our distributions to:
\begin{align*}
    F &\to \mathcal{N}(0, \sigma^2 I_d) \\
    F' &\to \mathcal{N}(\mu' - \mu, \sigma^2 I_d)
\end{align*}

Next, we apply an orthogonal transformation (a rotation) $R$ to the coordinate system such that the vector $\mu' - \mu$ is aligned with the first standard basis vector. That is, $R(\mu' - \mu) = (\delta, 0, \ldots, 0)^T$, where $\delta = \|\mu' - \mu\|_2 \le \Delta$. Because the covariance matrix $\sigma^2 I_d$ is a multiple of the identity, the Gaussian distribution is rotationally invariant. Thus, applying $R$ yields:
\begin{align*}
    F &\to \mathcal{N}(0, \sigma^2 I_d) \\
    F' &\to \mathcal{N}((\delta, 0, \ldots, 0)^T, \sigma^2 I_d)
\end{align*}

Because the covariance matrix is diagonal, these multivariate distributions can be factored into a product of independent 1-dimensional marginals:
\begin{align*}
    F &= \mathcal{N}(0, \sigma^2) \otimes \mathcal{N}(0, \sigma^2)^{\otimes d-1} \\
    F' &= \mathcal{N}(\delta, \sigma^2) \otimes \mathcal{N}(0, \sigma^2)^{\otimes d-1}
\end{align*}

The TV distance between two product distributions that are identical in all but one component is simply the TV distance of that differing component. Therefore, we have successfully reduced the problem to one dimension:
\begin{equation}
    TV(F, F') = TV\Big(\mathcal{N}(0, \sigma^2), \mathcal{N}(\delta, \sigma^2)\Big)
\end{equation}

Assume without loss of generality that $\delta > 0$ (if $\delta = 0$, the distance is trivially 0). Let $p(x)$ and $q(x)$ be the probability density functions of $\mathcal{N}(0, \sigma^2)$ and $\mathcal{N}(\delta, \sigma^2)$ respectively:
\begin{align*}
    p(x) &= \frac{1}{\sqrt{2\pi}\sigma} \exp\left(-\frac{x^2}{2\sigma^2}\right) \\
    q(x) &= \frac{1}{\sqrt{2\pi}\sigma} \exp\left(-\frac{(x-\delta)^2}{2\sigma^2}\right)
\end{align*}

The TV distance for continuous distributions is defined as:
\begin{equation}
    TV(p, q) = \frac{1}{2} \int_{-\infty}^\infty |p(x) - q(x)| \, dx
\end{equation}

To evaluate this integral without absolute values, we find the point where $p(x) = q(x)$. Because both distributions have the same variance, their densities cross exactly once, at the midpoint of their means: $x = \delta / 2$.
For $x < \delta / 2$, $x$ is closer to $0$ than to $\delta$, meaning $p(x) > q(x)$.
For $x > \delta / 2$, $p(x) < q(x)$. We can therefore rewrite the TV distance as:
\begin{equation}
    TV(p, q) = \int_{-\infty}^{\delta/2} \big( p(x) - q(x) \big) \, dx
\end{equation}

Let $\Phi(z)$ represent the cumulative distribution function (CDF) of the standard normal distribution $\mathcal{N}(0,1)$. Evaluating the integrals gives:
\begin{align*}
    \int_{-\infty}^{\delta/2} p(x) \, dx &= \Phi\left(\frac{\delta/2}{\sigma}\right) = \Phi\left(\frac{\delta}{2\sigma}\right) \\
    \int_{-\infty}^{\delta/2} q(x) \, dx &= \Phi\left(\frac{\delta/2 - \delta}{\sigma}\right) = \Phi\left(-\frac{\delta}{2\sigma}\right)
\end{align*}

Since the standard normal distribution is symmetric, $\Phi(-z) = 1 - \Phi(z)$. Thus:
\begin{align}
    TV(p, q) &= \Phi\left(\frac{\delta}{2\sigma}\right) - \left[1 - \Phi\left(\frac{\delta}{2\sigma}\right)\right] \nonumber \\
             &= 2\Phi\left(\frac{\delta}{2\sigma}\right) - 1 \nonumber \\
             &= 2 \left( \Phi\left(\frac{\delta}{2\sigma}\right) - \frac{1}{2} \right) \nonumber \\
             &= 2 \left( \Phi\left(\frac{\delta}{2\sigma}\right) - \Phi(0) \right)
\end{align}

By the Mean Value Theorem, there exists some $c \in (0, \frac{\delta}{2\sigma})$ such that:
\begin{equation*}
    \Phi\left(\frac{\delta}{2\sigma}\right) - \Phi(0) = \frac{\delta}{2\sigma} \cdot \phi(c)
\end{equation*}
where $\phi$ is the standard normal PDF. Because $\phi(z)$ reaches its absolute maximum of $\frac{1}{\sqrt{2\pi}}$ at $z=0$, we have $\phi(c) < \frac{1}{\sqrt{2\pi}}$. 

Plugging this upper bound back into our TV distance equation yields:
\begin{equation}
    TV(p, q) < 2 \left( \frac{\delta}{2\sigma} \cdot \frac{1}{\sqrt{2\pi}} \right) = \frac{\delta}{\sqrt{2\pi}\sigma}
\end{equation}

Finally, since we are given that $\delta \le \Delta$, we conclude:
\begin{equation}
    TV(F, F') \le \frac{\Delta}{\sqrt{2\pi}\sigma} = O\left(\frac{\Delta}{\sigma}\right)
\end{equation}
This completes the proof.
\end{proof}

\section{Follow The Perturbed Leader with Gaussian Noise}\label{app:gaussian}

Our use of self-concordant noise is essential for ensuring that the predictions made by the learner are valid probability distributions. However, many loss functions are still perfectly well defined for ``improper'' distributions $P_t$ lying outside the simplex (e.g., that have negative probabilities or probabilities that do not sum to one). For example, the squared loss $\ellsq(p, y) = \|p-y\|^2$ is well-defined for all $p \in \mathbb{R}^K$ (not just $p \in \Delta_K$). In this appendix, we show that if we allow the learners to make improper predictions, it is possible to replace the self-concordant noise in Algorithm~\ref{alg:multiclass} with simple Gaussian noise and obtain a slightly better dependence on $K$. 

To do so, we must work with extended variants of our previous classes of losses. We define an \emph{extended loss} $\ell$ to be any function $\ell: \mathbb{R}^K \times \{e_1, e_2, \dots, e_K\} \rightarrow \mathbb{R}$, and an \emph{extended proper loss} to be any $\ell$ satisfying $p \in \argmin_{p' \in \mathbb{R}^K}\E_{Y \sim p}[\ell(p', Y)]$ for all $p \in \Delta_K$. We will say an extended proper loss is \emph{bounded} if $|\ell(p, y)| \leq 1$ for all $y$ and for all $p \in \mathbb{R}^{K}$. Similarly, we say an extended proper loss is \textit{$\beta$-smooth} if it satisfies the $\beta$-smooth relation for all $p, q \in \mathbb{R}^K$. Let $\bar{\mathcal{L}}$ denote the set of all bounded extended proper losses, and let $\bar{\mathcal{S}}_{\beta}$ denote the set of all $\beta$-smooth proper losses.

\begin{algorithm}[ht]
   \caption{Multiclass U-Calibration With Gaussian Noise}
   \label{alg:multiclass_gaussian}
   {\bfseries Initialize:} $\sigma > 0$, $\pbar_{0}$ is the uniform distribution.
   
   \For{$t=1, \ldots, T$}{

   Sample $Z_t \sim \mathcal{N}(0, \sigma^2 I_K)$.

   Predict $P_t = \pbar_{t-1} + Z_t$.

   Observe label $y_t$ and update the empirical average $\pbar_{t} = \frac{1}{t}\sum_{\tau=1}^{t} y_\tau$.

   }
\end{algorithm}

In the following theorem, we show that Follow-The-Perturbed-Leader with Gaussian noise (Algorithm~\ref{alg:multiclass_gaussian}) obtains $O(\beta\log T)$ regret for $\beta$-smooth extended proper losses and $O(\sqrt{KT \log T})$ regret for bounded extended proper losses.

\begin{theorem}\label{thm:multiclass_gaussian}
Against an oblivious adversary, the Follow-The-Perturbed-Leader algorithm with variance $\sigma^2$ Gaussian noise (Algorithm~\ref{alg:multiclass_gaussian}) guarantees the following worst-case regret bounds:

\begin{enumerate}
    \item For any $\beta$-smooth extended proper loss $\ell \in \bar{\mathcal{S}}_{\beta}$, $\Reg_{\ell} = O(\beta K\sigma^2T + \beta\log T)$.
    \item For any bounded extended proper loss $\ell \in \bar{\mathcal{L}}$, $\Reg_{\ell} = O(\sigma^{-1}\log T + \sigma KT)$.
\end{enumerate}

\noindent
Choosing $\sigma = \sqrt{(\log T)/(KT)}$, we obtain the guarantees:

\begin{enumerate}
    \item For any $\beta$-smooth extended proper loss $\ell \in \bar{\mathcal{S}}_{\beta}$, $\Reg_{\ell} = O(\beta\log T)$.
    \item For any bounded extended proper loss $\ell \in \bar{\mathcal{L}}$, $\Reg_{\ell} = O(\sqrt{KT \log T})$.
\end{enumerate}

\end{theorem}

Establishing the first part of the theorem is a straightforward modification of Proposition~\ref{prop:multiclass_smooth} from our existing analysis for self-concordant noise.

\begin{proposition}\label{prop:multiclass_smooth_gaussian}
For any $\beta>0$ and any extended loss function $\ell\in \bar{\calS}_\beta$, \pref{alg:multiclass_gaussian} achieves 
$\Reg_\ell = O(\beta\log T + \beta K \sigma^2 T)$.
\end{proposition}
\begin{proof}
As before, according to~\citet{luo2024optimal}, the regret of the FTL strategy (that is, predict $\pbar_{t-1}$ at time $t$) is $O(\beta\log T)$. 

It thus remains to analyze the difference between the loss of \pref{alg:multiclass_gaussian} and that of FTL, i.e., $\E\sbr{\sum_{t=1}^T \ell(P_t, y_t) - \ell(\pbar_{t-1}, y_t)}$.
To do so, we plug in the definition of $P_t=\pbar_{t-1}+Z_t$ and use the smoothness property:
\begin{align*}
\E_{Z_t}\sbr{\ell(P_t, y_t) - \ell(\pbar_{t-1}, y_t)}
&\leq \E_{Z_t}\sbr{\inner{\nabla\ell_p(\pbar_{t-1}, y_t), Z_t}} + \frac{1}{2}\beta \E_{Z_t}\sbr{\|Z_t\|_2^2} \\
&=  \frac{1}{2}\beta \E_{Z_t}\sbr{\|Z_t\|_2^2} && (\E_{Z_t}\sbr{Z_t} = 0) \\
&= \beta K \sigma^2 / 2 .
\end{align*}

\noindent
Summing over $t$ finishes the proof.
\end{proof}

As before, to show the second part of the theorem, we will first relate the regret of FTPL to BTPL (establishing stability), and then bound the regret of BTPL. To bound the gap between the regret of FTPL and BTPL, it is enough to note that the TV distance between the actions taken by FTPL and BTPL is small, and so there exists a coupling between their actions with low disagreement. 

\begin{lemma}[Stability]\label{lem:stability}
Fix an oblivious adversary and any extended proper loss $\ell \in \bar{\mathcal{L}}$. For each $t\in[T]$, let $P_t \sim \mathcal{N}(\bar{p}_{t-1}, \sigma^2 I_K)$ and $P_{t+1} \sim \mathcal{N}(\bar{p}_{t}, \sigma^2 I_K)$. Then
\begin{align*}
\E\sbr{\sum_{t=1}^T \ell(P_t, y_t) - \ell(P_{t+1}, y_t)} = O(\sigma^{-1}\log T).
\end{align*}
\end{lemma}
\begin{proof}
Fix an adversarially chosen sequence of outcomes $y_1, y_2, \dots, y_T$.
Let $F_t$ and $F_{t+1}$ denote the laws of $P_t$ and $P_{t+1}$ respectively. By the same total-variation stability argument as in Lemma~\ref{lem:ellipsoid-stability-new}, we have
\begin{align*}
\E\sbr{\ell(P_t, y_t) - \ell(P_{t+1}, y_t)} \leq 2\cdot \|F_{t+1} - F_t\|_{\mathrm{TV}} .
\end{align*}
Since $\| \bar{p}_{t} - \bar{p}_{t-1} \| \leq 2/t$, Lemma~\ref{lem:gaussian_tv} gives $\|F_{t+1} - F_t\|_{\mathrm{TV}} = O(1/(\sigma t))$. 
Summing over $t \in [T]$ yields $\sum_{t=1}^T \|F_{t+1} - F_t\|_{\mathrm{TV}} = O(\sigma^{-1}\log T)$, and thus establishes the theorem statement.
\end{proof}

We next bound the regret of BTPL.

\begin{lemma}[Regret of BTPL]\label{lem:BTPL-direct-extended}
Against an oblivious adversary, for any extended proper loss $\ell \in \bar{\mathcal{L}}$, the predictions $\{P_{t+1}\}_{t=1}^T$ of Algorithm~\ref{alg:multiclass_gaussian} satisfy
\begin{align*}
\sum_{t=1}^T \bigl( \ell(P_{t+1}, y_t) - \ell(\bar{p}_{T}, y_t) \bigr) = O(\sigma K T).
\end{align*}
\end{lemma}
\begin{proof}
Applying Lemma~\ref{lem:BTPL-direct}, we have that 
\begin{align*}
\sum_{t=1}^T \bigl( \ell(P_{t+1}, y_t) - \ell(\bar{p}_{T}, y_t) \bigr)  \leq 2 \sum_{t=1}^T \|t Z'_{t+1} - (t-1) Z'_{t}\|_1 
\end{align*}
for any sequence of random variables $Z'_{t}$ such that each $Z'_{t}$ has the same distribution as $Z_{t}$ (i.e., any coupling of the $Z_{t}$). We will choose $Z'_{t}$ by sampling a single noise vector $n \sim N(0, \sigma^2 I_K)$ and letting $Z'_{t} = n$ for all $t$. We therefore have
\begin{align*}
\sum_{t=1}^T \|t Z'_{t+1} - (t-1) Z'_{t}\|_1 = T \|n\|_{1}.
\end{align*}

Taking expectation over $n$ and using $\E\sbr{\|n\|_1} = O(K\sigma)$ results in the desired claim.
\end{proof}

We can now combine the above lemmas to prove Theorem~\ref{thm:multiclass_gaussian}.

\begin{proof}[Proof of Theorem~\ref{thm:multiclass_gaussian}]
Part~(1) is exactly Proposition~\ref{prop:multiclass_smooth_gaussian}.

For Part~(2), fix any $\ell \in \bar{\mathcal{L}}$. Let $\{P_t\}_{t=1}^T$ denote the predictions of $\FTPL$ and let $\{P_{t+1}\}_{t=1}^T$ denote the corresponding predictions of $\BTPL$ (as in Lemma~\ref{lem:stability}). Then
\begin{align*}
\E\sbr{\sum_{t=1}^T \bigl( \ell(P_t, y_t) - \ell(\bar p_T, y_t) \bigr)}
&= \E\sbr{\sum_{t=1}^T \bigl( \ell(P_t, y_t) - \ell(P_{t+1}, y_t) \bigr)}
 + \E\sbr{\sum_{t=1}^T \bigl( \ell(P_{t+1}, y_t) - \ell(\bar p_T, y_t) \bigr)} \\
&\le O(\sigma^{-1}\log T) + O(\sigma K T),
\end{align*}
where we applied Lemma~\ref{lem:stability} and Lemma~\ref{lem:BTPL-direct-extended}.
Choosing $\sigma = \sqrt{(\log T)/(K T)}$ yields $\Reg_{\ell} = O(\sqrt{K T\log T})$.
\end{proof}

\end{document}